\documentclass[twoside,11pt]{article}

\usepackage{blindtext}

\usepackage[abbrvbib, preprint]{jmlr2e}

\usepackage{xcolor}
\usepackage{amsmath}
\usepackage{dsfont}
\usepackage{mathtools}
\usepackage{bm}
\usepackage{booktabs}
\usepackage{graphicx}
\usepackage[createShortEnv]{proof-at-the-end}

\newcommand{\indep}{\perp \!\!\! \perp}
\newcommand{\notindep}{\not\!\perp\!\!\!\perp}

\newcommand{\note}[1]{
}

\usepackage{lastpage}
\jmlrheading{23}{2026}{1-\pageref{LastPage}}{1/21; Revised 5/22}{9/22}{21-0000}{Marco Simnacher, Georg Keilbar, Benjamin König, Christoph Lippert, Sonja Greven}

\ShortHeadings{CI Testing for LLMs}{Simnacher, Keilbar, König, Lippert, Greven}
\firstpageno{1}

\begin{document}

\title{Embedded Conditional Independence Tests for Large Language Model Generated Text with an Application to German Parliament Speeches}

\author{\name Marco Simnacher \email marco.simnacher@hu-berlin.de \\
        \name Georg Keilbar \email georg.keilbar@hu-berlin.de \\
        \name Benjamin König \email b.koenig@hu-berlin.de \\
       \addr Chair of Statistics\\
       Humboldt-Universität zu Berlin\\
       Spandauer Str. 1, 10178 Berlin, Germany
       \AND
        \name Christoph Lippert \email Christoph.Lippert@hpi.de \\
       \addr Chair Digital Health \& Machine Learning\\
       Hasso-Plattner-Institut for Digital Engineering\\
       Prof.-Dr.-Helmert-Straße 2-3, 14482 Potsdam, Germany
        \AND
       \name Sonja Greven \email sonja.greven@hu-berlin.de\\
       \addr Chair of Statistics\\
       Humboldt-Universität zu Berlin\\
       Spandauer Str. 1, 10178 Berlin, Germany}

\editor{My editor}

\maketitle

\begin{abstract}
Conditional independence tests (CITs) test for conditional dependence between two random objects $X$ and $Y$ given a third random object $Z$.
Existing CITs have limited applicability to high-dimensional data, especially multimodal data like text.
However, we show that such tests are of interest for large language model (LLM) outputs, where we test whether an output $X$ generated from a source text $Z$ carries information about an attribute $Y$ beyond $Z$ itself.
For this purpose, we propose embedded CITs (eCITs), which embed $X$ and $Z$ and apply an existing CIT to the resulting representations and to $Y$.
We show that, provided the embedding of $Z$ is sufficient, i.e.\ retains the information $Z$ carries about either $Y$ or the representation of $X$, the null hypothesis transfers from $X$ and $Z$ to their representations, so that a CIT valid for the embedded hypothesis is valid for the original one.
We further give conditions for equivalence of the two hypotheses, and show that sufficiency weakens to mean sufficiency when the embedded test targets conditional mean independence.
We propose a semi-synthetic simulation design to assess type I error (T1E) control and power of the eCITs for given embedding maps on a specific dataset and task, and use it to evaluate them on our application.
Applying the eCITs to German Parliament speeches, we find for all combinations of embedding maps considered that the summaries of two LLMs contain information about the speaker's faction and gender beyond the speech they were generated from. 
\end{abstract}


\section{Introduction}\label{sec:intro}

A conditional independence test (CIT) is a statistical test for the null hypothesis of conditional independence (CI), $H_0:X\indep Y|Z$, between two random variables, $X$ and $Y$, given a third variable, $Z$, against the alternative hypothesis of conditional dependence, $H_1:X\notindep Y|Z$. $X$ and $Y$ are conditionally dependent, given $Z$, if $X$ (or $Y$) provides additional information about $Y$ (or $X$), given $Z$.
A CIT should control the type I error (T1E), i.e. the probability of falsely rejecting the null hypothesis, while simultaneously achieving high power, i.e. a high probability of correctly rejecting the null hypothesis of CI. 

One setting in which such conditional independencies are of interest is the evaluation of large language models (LLMs). Consider a source text $Z$, e.g. a parliament speech, observed together with a protected attribute $Y$, e.g. the speaker's gender or faction, and an LLM-generated summary $X$ of the source text.
Such summaries can contain biases regarding the protected attribute, framing e.g. the speeches of one faction more critically than those of another, although the content of the speeches does not differ.
Due to the widespread use of LLMs across different tasks, such biases are widely studied \citep{gallegos2024bias}.
However, the complexity of LLMs and their task-specific behavior make general statements about the bias of a specific LLM challenging \citep{anthis2024impossibility}, and existing bias metrics are of limited use in this setting.
Intrinsic metrics based on embeddings or token probabilities often do not correlate with the bias observed in downstream tasks \citep{goldfarb2020intrinsic, delobelle2022measuring}.
Extrinsic metrics are typically bound to benchmark datasets derived from sentence completion or occupation masking, so that their results need not translate to domain-specific tasks \citep{li2023survey}.
LLM-based evaluators are applicable to arbitrary tasks, but can themselves be biased \citep{sun2022bertscore, gao2025llm}, and their scores do not come with statistical guarantees.

We instead define bias as additional information in the LLM output $X$ about the protected attribute $Y$ beyond the information already contained in the source text $Z$. This directly corresponds to the alternative hypothesis of a CIT, so bias can be tested on a given task and dataset with a T1E guarantee.
Since the summary is generated solely from the source text, any such additional information cannot come from the observed input itself, but must have entered through the model, for example via memorization of the source texts through the training data the model was fitted on.
The definition is reference-free in the sense that no gold-standard output or counterfactually generated output has to be constructed, since the source text itself supplies the reference. While we focus on summarization as one instance, the definition applies to any task in which the output is generated from a source text.
In our empirical application with a categorical variable $Y$, the definition is a conditional version of the partisanship measure of \citet{gentzkow2019measuring}, who quantify how strongly congressional speech reveals a speaker's faction. We ask the same of the summaries, treating the speech as the conditioning variable: whether the summary reveals the speaker's faction beyond the speech it was generated from.

Testing this hypothesis, however, is not straightforward when $X$ and $Z$ are texts.
Existing CITs focus on low-dimensional, vector-valued variables \citep{li2020nonparametric}, and both $Z$ and $X$ cause problems through their structure and high-dimensionality. 
The structure of $Z$ affects the T1E of the CITs. Their null hypothesis comparisons require an object constructed from $Z$: the conditional distribution of $X$ or of $Y$ given $Z$ for resampling-based comparisons \citep{candes2018panning,berrett2020conditional}, a kernel or a partition on the support of $Z$ for kernel- and neighborhood-based comparisons \citep{zhang2011kernel, strobl2019approximate}, or a regression on $Z$ for residual-based comparisons \citep{shah2020hardness,lundborg2024projected}.
None of these objects can be constructed on raw texts with the guarantees the tests require, whether because the construction itself is unavailable or because the rate conditions on which validity rests are not. The tests thus become applicable only on a feature representation of $Z$ and condition on that representation rather than on $Z$, as in the application of algorithm-agnostic significance tests to multimodal data \citep{kook2024algorithm}.
The text structure of $X$ affects both whether a CIT can be applied at all and, once it can, its power. \citet{simnacher2026deep} address this for structured, high-dimensional $X$ by embedding it onto a feature representation and applying a CIT to the representation instead. Provided the embedding map is learned as they prescribe, e.g.\ on external data or through sample splitting, this leaves the T1E unaffected in their setting, where $Z$ is not embedded, and the choice of representation affects only the power, which depends on whether the representation retains the conditional association between $X$ and $Y$ given $Z$.
The representation, however, remains high-dimensional, so that only a few CITs are applicable to it. 
Neither line of work addresses the representation of $Z$, on which the T1E of the resulting test depends.

We therefore require CITs that are applicable to texts $X$ and $Z$, control the T1E for the hypothesis stated on the level of the texts, and have power against relevant alternatives.
To meet these requirements, we propose embedded CITs (eCITs), which map both $Z$ and $X$ onto feature representations and apply an existing CIT for conditional dependence of the representation of $X$ and $Y$, given the representation of $Z$, making the test applicable to structured, high-dimensional inputs.
Since the CIT conditions on the representation of $Z$ rather than on the source
text, it is enough that that representation retain the information in $Z$ about
either $Y$ or the representation of $X$. Any information it drops plays the role of an omitted conditioning variable and can inflate the T1E.
We call embedding maps $\omega_Z$ yielding such a representation sufficient and characterize them in Subsection~\ref{subsec:suff_embedding_maps}.
The representation of $X$ needs no sufficiency condition of its own and mainly affects the power, since the test has power only if it retains the conditional association between $X$ and $Y$ given $Z$ in a form the CIT can detect.
The two choices nevertheless interact, since the sufficiency of the embedding of $Z$ is stated relative to the representation of $X$, unlike in \citet{simnacher2026deep}, where the conditioning variable is not embedded and the choice of representation affects the power alone.
We discuss this further in Subsection~\ref{subsec:suff_embedding_maps} and Section~\ref{sec:application}.
The CIT contributes its guarantees at the level of the representations, so whether they transfer to the texts depends on the embedding maps and the dataset at hand.

Our contributions to the existing literature are as follows.
First, we justify CI testing as a bias criterion for LLM outputs and discuss when the null and the alternative can hold for a fixed LLM (Section~\ref{sec:motivation_llm_summaries}).
Second, we propose eCITs, which combine embedding maps for $X$ and $Z$ with existing CITs, making CI testing applicable when both $X$ and $Z$ are texts (Section~\ref{sec:theory}).
Third, we characterize sufficient embedding maps for $Z$, under which the null hypothesis stated on the texts transfers to the null hypothesis stated on the representations, so that a valid CIT for the embedded hypothesis controls the T1E for the textual one. We also give conditions under which the two hypotheses are equivalent, and show that the requirement weakens to mean sufficiency when the embedded test targets conditional mean independence (Subsection~\ref{subsec:suff_embedding_maps}).
Fourth, we adapt the projected covariance measure \citep[PCM;][]{lundborg2024projected}, the conditional mean independence test (CMIT) we apply, to the categorical $Y$ of our application, with a precision-weighted projection direction that we show to be optimal, and record the conditions under which the validity argument of \citet{lundborg2024projected} carries over (Appendix~\ref{app:catpcm}).
Fifth, we show empirically how insufficient embeddings of $Z$ inflate the T1E, and provide a simulation design that can be used to assess the eCITs on a given dataset before applying them (Subsection~\ref{subsec:simulation}).
On speech-summary pairs, eCITs with a sufficient embedding of $Z$ are calibrated at all sample sizes considered for the categorical outcomes, also for high-dimensional representations, while for insufficient ones the inflation ranges from negligible to growing steeply with the sample size, depending on how much of the dropped information the remaining embedding maps recover.
Finally, we apply the eCITs to German Parliament speeches and find, for both summarizers and all combinations of embedding maps considered, that the summaries of two LLMs contain information about the speaker's faction and gender beyond the speech (Subsection~\ref{subsec:real_world_application}).

\section{Conditional (In)dependence for LLM-Generated Summaries}\label{sec:motivation_llm_summaries}

Along with a protected scalar variable $Y$ and a source text $Z$, we generate LLM outputs $X$ from the source text $Z$. We define bias as additional information in the output $X$ about a protected attribute $Y$ given the source text $Z$. Specifically, we test for bias in an output by testing the hypothesis
\begin{align}\label{equ:hypothesis_text}
    H_0: Y\indep X \mid Z \quad \text{vs.} \quad H_1: Y\notindep X\mid Z.
\end{align}
Under the null, the output does not add any information about the protected attribute beyond the information already available in the source text. Under the alternative, the output provides additional information about the protected attribute that was not available in the source.

We distinguish between the text corpus $D_{LLM}$ used to train the LLM and the test dataset $D_{CIT}=(X_i, Y_i,Z_i)_{i=1}^n$ used in the CIT. 
To generate $X_i$, we prompt the LLM with a task constant across all $i$ and a source text $Z_i$, cf. Subsection~\ref{subsec:data} and Appendix~\ref{app:summary_gen_details} for details. 
If the text corpus $D_{LLM}$ and $(Z,Y)\in D_{CIT}$ are independent, we are trivially under the null hypothesis in \eqref{equ:hypothesis_text}, since the output $X$ is a measurable function of $Z$, $D_{LLM}$ and, under stochastic decoding, of sampling noise drawn independently of $(Y,Z)$, and hence conditionally independent of $Y$ given $Z$.
If $Z,Y$ or texts $\Tilde{Z}$ with the same $Y$ are contained in the text corpus, then the null or alternative hypothesis can be true.

A plausible mechanism for such a conditional dependence is memorization. This is also the situation we expect in our application in Section~\ref{sec:application}, where the source texts are German Parliament speeches and the two LLMs used to summarize the speeches are pre-trained on web-scale German corpora that plausibly include Bundestag plenary protocols. If the LLM has seen a source text $Z_i$, or other texts by the same speaker, during training, its weights can carry information about the associated attribute $Y_i$ that $Z_i$ alone does not determine. When generating $X_i$ from $Z_i$, the model may then reproduce phrasing or framing tied to that particular observation rather than to the content of $Z_i$, so that $X$ reveals $Y$ beyond $Z$. The following example isolates this mechanism in a setting in which the null and the alternative can be constructed by design.

\begin{figure}[t!]
    \centering
    \includegraphics[width=0.7\linewidth]{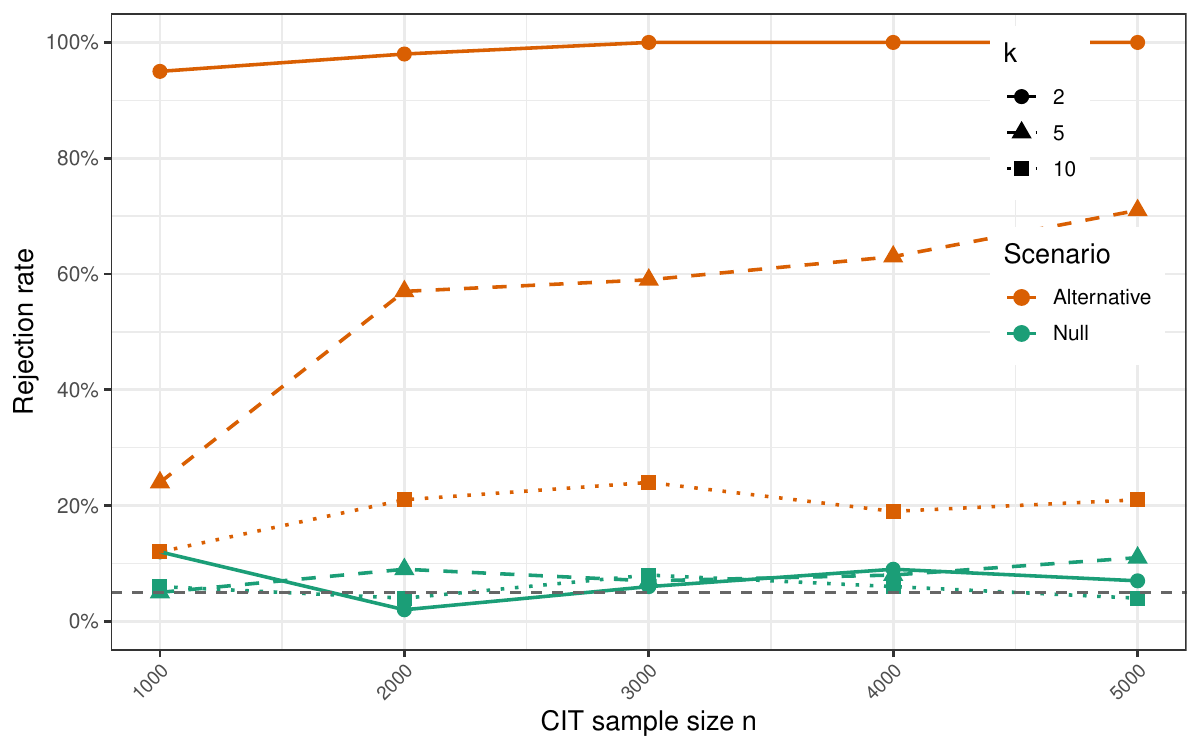}
    \caption{Rejection rates of the PCM testing conditional mean independence of $X$ and $Y$ given $Z$ with default parameters applied to data generated with $Z\sim N(0_{10}, I_{10})$, $Y_i=Z_{i,1}+\varepsilon_i, \varepsilon_i\overset{iid}{\sim} N(0,1)$, and $X_i=f_{\hat\theta}(Z_i)$, where $f_{\hat\theta}$ is a knn model predicting $Y$ from $Z$ with $k=2,5,10$. Under the null, the training dataset $D_{tr}=(Z_i,Y_i)_{i=1}^N$ is independent of the test dataset $D_{test}=(X_i,Y_i,Z_i)_{i=1}^n$, under the alternative $D_{test}\subset D_{tr}$.}
    \label{fig:knn_bias_example}
\end{figure}
\begin{example}\label{ex:knn}
We generate $D_{tr}=(Z_i,Y_i)_{i=1}^N,$ with independent copies $Z_i\sim N(0_{10},I_{10})$ and $Y_i=Z_{i,1}+\varepsilon_i,\  \varepsilon_i\overset{iid}{\sim} N(0,1)$.  
We fit a k-nearest neighbor model $f_\theta$ on  $D_{tr}=(Z_i,Y_i)$, and afterwards fix the model $f_{\hat\theta}$. 
Then, we predict $Y_i$ from $Z_i$ using $ (Z_i,Y_i)_{i=1}^n\subset D_{test}$.  
Under the null hypothesis, $(Z_i,Y_i)_{i=1}^n$ from the test dataset $D_{test}=(X_i,Y_i,Z_i)_{i=1}^n$ is sampled independently of $D_{tr}$,  $X_i=f_{\hat\theta}(Z_i)$ are out-of-sample predictions, and thus, $X\indep Y\mid Z$ holds on $D_{test}$. Under the alternative, $D_{test}\subset D_{tr}$, the predictions $X_i=f_{\hat\theta}(Z_i)$ are in-sample, and thus, $X\notindep Y\mid Z$ holds on $D_{test}$.

Figure~\ref{fig:knn_bias_example} shows the rejection rates of the PCM with default parameters testing for conditional dependence for three k-nearest neighbor (knn) models (k=2,5,10) under the null and alternative. 
Under the null, the T1E of the PCM is within one Monte Carlo standard error of the nominal level of 0.05. The power under the alternative depends on the number of nearest neighbors. For a specific $Y_i$, this $Y_i$ contributes more to the model's predictions for a model with smaller $k$, thus leading to larger conditional dependence and higher rejection rates. 
In the example, $D_{tr}$ and $D_{test}$ mimic $D_{LLM}$ and $D_{CIT}$. We observe that a model with a stronger memorization of individual observations leads to stronger conditional dependence, detectable by the PCM.
\end{example}

\section{Embedded Conditional Independence Tests}\label{sec:theory}

Let $(\mathcal{X}, \mathcal{F}_\mathcal{X}), (\mathcal{Y}, \mathcal{F}_\mathcal{Y}), (\mathcal{Z}, \mathcal{F}_\mathcal{Z})$ be measurable spaces and $X,Y,Z$ be $(\mathcal{X}, \mathcal{F}_\mathcal{X})$-, $(\mathcal{Y}, \mathcal{F}_\mathcal{Y})$-, $(\mathcal{Z}, \mathcal{F}_\mathcal{Z})$-valued random variables with joint distribution $\mathbb{P}^{X,Y,Z}$. Moreover, let $\hat{\beta}$ and $\hat{\gamma}$ be $(\mathcal{B},\mathcal{F}_{\mathcal{B}})$- and $(\mathcal{C},\mathcal{F}_{\mathcal{C}})$-valued random variables representing estimated parameters of embedding maps $\omega_X:\mathcal{X}\times \mathcal{B}\to \mathfrak{X},(X,\hat{\beta})\mapsto \omega_X(X,\hat{\beta})$ and $\omega_Z:\mathcal{Z}\times \mathcal{C}\to \mathfrak{Z},(Z,\hat{\gamma})\mapsto \omega_Z(Z,\hat{\gamma})$. 
We denote the resulting random representations by $X^{\omega_X}=\omega_X(X,\hat{\beta})$ and $Z^{\omega_Z}=\omega_Z(Z,\hat{\gamma})$, which are $(\mathfrak{X},\mathcal{F}_\mathfrak{X})$- and $(\mathfrak{Z},\mathcal{F}_\mathfrak{Z})$-valued random variables.
The focus of this paper is on $\mathcal{X}$ and $\mathcal{Z}$ as spaces of text,  $\mathcal{Y}\subseteq\mathbb{R}$ or $\mathcal{Y}=\{1,\ldots,L\}$ as a univariate continuous or categorical variable, and $\mathfrak{X}\subseteq \mathbb{R}^q$ and $\mathfrak{Z}\subseteq \mathbb{R}^p$ as the spaces of feature representations. Nevertheless, the theoretical results apply also to CITs applicable to vector-valued $Y$, as well as general $X$ and $Z$ for which representations can be obtained. 
The notation and the assumptions on $\hat\beta$ follow \citet{simnacher2026deep}, where $Z$ is not embedded; we extend the setup by embedding $Z$ and by conditioning on $\hat\gamma$.

We assume that $(X_i,Y_i, Z_i), i=1,\hdots,n$ are $n$ independently and identically distributed (i.i.d.) copies of $(X,Y,Z)$. We define $X^{\omega_X}_i=\omega_X(X_i,\hat{\beta})$, and write $X^n=(X_1,\hdots,X_n)$, $(X,Y)^n=(X_i,Y_i)_{i=1}^n$, and analogously for all other combinations of $X_i^{\omega_X},$ $Z_i^{\omega_Z}$, $X_i$, $Y_i,$ $Z_i,$ $i=1,\hdots,n$.  
For the whole sample with text and their feature representations, we write $S=(X, Y, Z)^n$ and $S^\omega=(X^{\omega_X}, Y,Z^{\omega_Z})^n$, $S^{\omega_X}=(X^{\omega_X}, Y,Z)^n$, $S^{\omega_Z}=(X, Y,Z^{\omega_Z})^n$ respectively. 
Furthermore, let the underlying probability spaces for the samples be $(\mathcal{S}, \mathcal{F}_\mathcal{S}, \mathbb{P}^S)$ and analogously $(\mathcal{S}^\omega, \mathcal{F}_{\mathcal{S}^\omega}, \mathbb{P}^{S^\omega})$.
We call $\mathfrak X\times\mathcal Y\times\mathfrak Z$ the embedded observation space, in contrast to the embedded sample space $\mathcal S^{\omega_X}$ and $\mathcal S^\omega$, and denote the collections of null and alternative distributions of \eqref{equ:hypothesis_text}
by $
  \mathcal P_0=\big\{(\mathbb P^{X,Y,Z})^{\otimes n}\ :\ X\indep Y\mid Z\ \text{under}\ \mathbb{P}^{X,Y,Z}\big\},$ and $\mathcal P_1=\big\{\mathbb (P^{X,Y,Z})^{\otimes n}\ :\ X\notindep Y\mid Z\ \text{under}\ \mathbb{P}^{X,Y,Z}\big\},
$
respectively. 
On the embedded sample spaces we define the corresponding null classes on $\mathcal{S}^\omega $ and $\mathcal{S}^{\omega_X} $ as 
\begin{align}
\mathcal{P}^\omega_0 &:= \bigl\{Q \text{ probability measure on }
 (\mathcal{S}^\omega,\mathcal{F}_{\mathcal{S}^\omega}) :\
 X^{\omega_X,n}\indep Y^n \mid Z^{\omega_Z,n} \text{ under } Q \bigr\},
 \label{eq:P0omega}
 \\
\mathcal{P}^{\omega_X}_0 &:= \bigl\{Q \text{ probability measure on }
 (\mathcal{S}^{\omega_X},\mathcal{F}_{\mathcal{S}^{\omega_X}}) :\
 X^{\omega_X,n}\indep Y^n \mid Z^{n} \text{ under } Q \bigr\}.
 \label{eq:P0omegaX}
\end{align}
Both classes refer only to the embedded sample space and carry no reference to $\hat\beta$, $\hat\gamma$ or $\mathbb{P}^{X,Y,Z}$. 
Their elements are not required to be product measures: this is needed because the conditional law of $S^\omega$ given the embedding parameters is in general not a product measure when the parameters are estimated in sample. 
Both classes depend on $n$, which we suppress in the notation.

Throughout, we assume that $(\mathcal{X},\mathcal{F}_\mathcal{X})$, $(\mathcal{Y},\mathcal{F}_\mathcal{Y})$, $(\mathcal{Z},\mathcal{F}_\mathcal{Z})$ and $(\mathfrak{X},\mathcal{F}_\mathfrak{X})$, $(\mathfrak{Z},\mathcal{F}_\mathfrak{Z})$ are standard Borel. 
Since finite products of standard Borel spaces are standard Borel, this carries over to the sample spaces $\mathcal{S}$, $\mathcal{S}^{\omega_X}$ and $\mathcal{S}^{\omega}$, and guarantees the existence of the regular conditional distributions used below. 
The assumption is satisfied in our setting: $\mathcal{X}$ and $\mathcal{Z}$ are the countable sets of finite strings over a finite alphabet equipped with the discrete $\sigma$-field, while $\mathcal{Y}\subseteq\mathbb{R}$ or $\mathcal{Y}=\{1,\ldots,L\}$ and $\mathfrak{X}\subseteq\mathbb{R}^q$, $\mathfrak{Z}\subseteq\mathbb{R}^p$ carry their Borel $\sigma$-fields. 
No such assumption is placed on the parameter spaces $(\mathcal{B},\mathcal{F}_\mathcal{B})$ and $(\mathcal{C},\mathcal{F}_\mathcal{C})$, which may be arbitrary measurable spaces.

To test hypothesis \eqref{equ:hypothesis_text}, we embed the texts with the embedding maps $\omega_X(\cdot,\hat\beta)$ and $\omega_Z(\cdot,\hat\gamma)$, and test instead
\begin{equation}\label{equ:hypothesis_embedding}
H^\omega_0 : \mathbb{P}^{S^\omega\mid\hat\theta}\in\mathcal{P}^\omega_0\
\mathbb{P}^{\hat\theta}\text{-a.s.}
\qquad\text{vs.}\qquad
H^\omega_1 : \mathbb{P}^{\hat\theta}\bigl(\mathbb{P}^{S^\omega\mid\hat\theta}
\in\mathcal{P}^\omega_0\bigr)<1,
\end{equation}
where $\hat\theta=(\hat\beta,\hat\gamma)$ and $\mathbb{P}^{S^\omega\mid\hat\theta}$ denotes a regular conditional distribution of the embedded sample given the embedding parameters. 

We propose the eCIT as the measurable function
\begin{equation}\label{eq:ecit}
\varphi_{\omega}:\mathcal{S}\times\mathcal{B}\times\mathcal{C}\times[0,1]\to\{0,1\},\qquad
(s,b,c,u)\mapsto\psi\bigl(\bar\omega(s,b,c),u\bigr),
\end{equation}
where $\bar\omega$ takes the whole sample and the estimators $\hat\beta,\hat\gamma$ as input and maps it onto $\mathcal{S}^{\omega}$ through the embedding maps, so that $\bar\omega(S,\hat\beta,\hat\gamma)=S^{\omega}$, and where $\psi:\mathcal{S}^{\omega}\times[0,1]\to\{0,1\}$ is a CIT testing \eqref{equ:hypothesis_embedding}, with $\varphi_{\omega}=0$ and $\varphi_{\omega}=1$ defining decisions for $H_0$ and $H_1$ in \eqref{equ:hypothesis_text}. 
The last argument is reserved for a random variable $U\sim\mathrm{Unif}[0,1]$, independent of $(S,\hat\beta,\hat\gamma)$, which carries any randomization employed by the CIT, such as a sample split \citep[cf. e.g.,][]{shah2020hardness}. 
Since a randomized test may always ignore this argument, deterministic CITs are included. 
We suppress $U$ from the notation, writing $\varphi_\omega(S)$ and $\psi(S^\omega)$, and let all probabilities of rejection be taken under the joint law of $(S,\hat\beta,\hat\gamma,U)$; analogously, $Q(\psi=1)$ abbreviates $(Q\otimes\mathrm{Unif}[0,1])(\psi=1)$ for $Q$ a probability measure on
$(\mathcal{S}^\omega,\mathcal{F}_{\mathcal{S}^\omega})$.
The test statistic of the eCIT is defined analogously as $T_{\varphi_\omega}(s,b,c,u)=T_{\psi}\bigl(\bar\omega(s,b,c),u\bigr)$ for a test statistic $T_{\psi}:\mathcal{S}^{\omega}\times[0,1]\to\mathbb{R}$ of the CIT $\psi$.

Throughout, the estimation procedure for the embedding parameters is fixed: it is a Markov kernel from $(\mathcal{S},\mathcal{F}_{\mathcal{S}})$ to $(\mathcal{B}\times\mathcal{C},\mathcal{F}_{\mathcal{B}}\otimes\mathcal{F}_{\mathcal{C}})$, specified for every $s\in\mathcal{S}$ and not depending on $\mathbb{P}^S$.
The parameters remain random, while their conditional distribution given $S=s$ is held fixed across the null class.
If $\hat\theta$ is a measurable function of the sample, as for the in-sample construction of $\omega_Z$ by regressing $X^{\omega_X,n}$ on $Z^n$, its conditional distribution given $S=s$ places probability one on the corresponding value. 
If $\hat\theta$ is instead learned on data independent of $S$, its conditional distribution does not depend on $s$ and carries the randomness of that fit. In general the conditional distribution both varies with $s$ and is non-degenerate given $S=s$, as when the sample is combined with external data or the fit is randomly initialized.
Since $U$ is independent of $(S,\hat\beta,\hat\gamma)$, the joint law of $(S,\hat\beta,\hat\gamma,U)$ is obtained from $\mathbb{P}^S$, this kernel and the law of $U$, and is therefore a function of $\mathbb{P}^S$ alone for a fixed estimation procedure. We write $\mathbb{P}^S(\varphi_\omega=1)$ for the rejection probability under it.

Both embedding maps $\omega_X$ and $\omega_Z$ have to be chosen in the eCIT. 
For the former, we follow the approaches discussed in \citet{simnacher2026deep}. This choice affects only the power of the test as long as the learning approaches characterized in \citet[][Corollary~3]{simnacher2026deep}, e.g. sample splitting or externally trained embedding maps, are used. 
The choice of $\omega_Z$ affects the T1E and the power of the eCIT. 
We will characterize sufficient embedding maps $\omega_Z$ in the next subsection. 
Furthermore, we evaluate the validity and power of the corresponding eCITs on semi-synthetic data in Subsection \ref{subsec:simulation}.

\subsection{Sufficient Embedding Maps}\label{subsec:suff_embedding_maps}

In this subsection, we take the embedding map $\omega_X$ as given, so that a valid test for \eqref{equ:hypothesis_embedding_X} is already a valid test for \eqref{equ:hypothesis_text}, and characterize the embedding maps $\omega_Z$ for which we may test \eqref{equ:hypothesis_embedding} instead. We first restate both hypotheses as CI statements conditioning on the estimated parameters $\hat\beta$ and $\hat\gamma$. We then give conditions under which $H_0^{\omega_X}$ implies $H_0^{\omega}$, so that a valid test of the embedded hypothesis controls the T1E for the hypothesis on the texts: $\hat\gamma$ must not introduce conditional dependence (Assumption~\ref{ass:gamma_no_cond_dep}), and $Z^{\omega_Z}$ must retain all information in $Z$ about either $Y$ or $X^{\omega_X}$ (Theorem~\ref{thm:text0_implies_emb0_Z}). We call such $\omega_Z$ sufficient. Finally, we give conditions for the converse implication (Theorem~\ref{thm:emb0_implies_text0_Z}) and for equivalence of the two hypotheses (Corollary~\ref{cor:equivalence_nulls}).

We can either embed $X$ or $Z$ first. We embed $X$ first, because this simplifies embedding $Z$, as we will see in the following sufficiency characterizations for the embedding map of $Z$. Thus, we assume throughout this subsection that $X$ was embedded onto $X^{\omega_X}$. Then, to test \eqref{equ:hypothesis_text}, we test instead
\begin{align}\label{equ:hypothesis_embedding_X}
    H_0^{\omega_X}:\mathbb{P}^{S^{\omega_X}| \hat{\beta}}\in \mathcal{P}_0^{\omega_X}\;\; \mathbb{P}^{ \hat{\beta}}\text{-a.s.}
    \quad vs.\quad 
    H_1^{\omega_X}:\mathbb{P}^{\hat{\beta}}(\mathbb{P}^{S^{\omega_X}| \hat{\beta}}\in \mathcal{P}_0^{\omega_X})<1, 
\end{align}
where $\mathbb{P}^{S^{\omega_X}|\hat{\beta}}$ denotes the conditional distributions of the sample after embedding $X$ given the parameter $\hat{\beta}$ of the embedding map $\omega_X$. 
We assume that $\omega_X$ was obtained as specified in \citet[][Theorem~1, Corollary~3]{simnacher2026deep}, e.g. learned on external data, through sample splitting or in an unsupervised manner. 
This ensures that a valid test for \eqref{equ:hypothesis_embedding_X} is a valid test for \eqref{equ:hypothesis_text}, although the embedding map can affect the power of the test.

Given such a $X^{\omega_X}$, we are interested in embedding maps $\omega_Z$ to test \eqref{equ:hypothesis_embedding} instead of \eqref{equ:hypothesis_embedding_X}. To do this, we will assume throughout that $\hat{\gamma}$ does not introduce conditional dependence. 
\begin{assumption}\label{ass:gamma_no_cond_dep}
    Either $a)\; \hat{\gamma}\indep Y^n\mid ((X^{\omega_X},Z)^n, \hat{\beta})$ or $b)\;\hat{\gamma}\indep X^{\omega_X,n}\mid ((Y,Z)^n, \hat{\beta})$.
\end{assumption}
This is similar to the assumption placed on the embedding map parameters $\hat{\beta}$ here and in the theoretical results of \citet{simnacher2026deep}. 
Furthermore, it holds trivially for $\hat{\gamma}$ learned via sample splitting, when $S$ denotes the holdout split, or on data independent of the test sample $S$, but has to be ensured for $\hat{\gamma}$ learned from $S$.
In particular, branch~a) holds by construction whenever $\hat\gamma$ is a measurable function of $(X^{\omega_X},Z)^n$, since $\hat\gamma$ is then measurable with respect to the conditioning $\sigma$-field. This covers the in-sample construction of $\omega_Z$ by regressing $X^{\omega_X,n}$ on $Z^n$, which is the natural use of the representation $X^{\omega_X}$ obtained in the first step.
In contrast to $\hat{\beta}$, this assumption is not sufficient for a valid test for \eqref{equ:hypothesis_embedding} to yield a valid test for \eqref{equ:hypothesis_embedding_X} and \eqref{equ:hypothesis_text}. Specifically, we need also that $Z$ does not contain information about $Y$ or $X^{\omega_X}$ beyond $Z^{\omega_Z}$, which we ensure below.

First, we show that the hypothesis $H_0^\omega$ in \eqref{equ:hypothesis_embedding}, formulated as an almost-sure statement of the conditional distribution $\mathbb{P}^{S^\omega\mid\hat\gamma,\hat\beta}$, is equivalent to a single conditional independence statement in which $\hat\gamma$ and $\hat\beta$ are added to the conditioning set.

\begin{lemmaE}[][end, text link={All proofs are provided in Appendix \ref{app:proofs}.}, text proof={Proof of Lemma~\ref{lem:disintegration}}]\label{lem:disintegration}
Let $\mathbb{P}^{S^\omega\mid\hat\gamma,\hat\beta}$ denote a regular conditional distribution of $S^\omega$ given $(\hat\gamma,\hat\beta)$. Then
\[
  X^{\omega_X,n} \indep Y^n \mid (Z^{\omega_Z,n},\hat\gamma,\hat\beta)
  \quad\Longleftrightarrow\quad
  \mathbb{P}^{S^\omega\mid\hat\gamma,\hat\beta} \in \mathcal{P}_0^\omega
  \quad \mathbb{P}^{\hat\gamma,\hat\beta}\text{-a.s.}
\]
The analogous equivalence holds for $S^{\omega_X}$, $\hat\beta$, $\mathcal{P}_0^{\omega_X}$ and $Z^n$ in place of $S^\omega, (\hat{\gamma},\hat\beta), \mathcal{P}_0^\omega$ and $Z^{\omega_Z,n}$. 
\end{lemmaE}
\begin{proofE}
Apply Lemma~\ref{lem:disintegration-general} with $V=S^\omega$, $W=(\hat\gamma,\hat\beta)$ and $(A,B,C)=(X^{\omega_X,n},Y^n,Z^{\omega_Z,n})$, so that $\mathcal{P}_0=\mathcal{P}_0^\omega$; and with $V=S^{\omega_X}$, $W=\hat\beta$ and $(A,B,C)=(X^{\omega_X,n},Y^n,Z^n)$, so that $\mathcal{P}_0=\mathcal{P}_0^{\omega_X}$, for the second statement.
\end{proofE}

\noindent
Both instances follow from the general Lemma~\ref{lem:disintegration-general} in
Appendix~\ref{app:proofs}. Note that no assumption is placed on the spaces
$\mathcal{B},\mathcal{C}$ in which $\hat\beta,\hat\gamma$ take their values, and
that $\mathcal{P}_0^\omega$ is not required to consist of product measures, so
that Lemma~\ref{lem:disintegration} applies unchanged when $\hat\gamma$ is
estimated in sample and the rows of $S^\omega$ are conditionally dependent.

Next, we show that for an embedding map $\omega_Z$ containing all information within $Z$ of either, $Y$ or $X^{\omega_X}$, the null hypothesis \eqref{equ:hypothesis_embedding_X} including texts $Z$ implies the null hypothesis \eqref{equ:hypothesis_embedding} including only the corresponding embeddings $Z^{\omega_Z}$. 
\begin{theoremE}[$Z$-text null implies $Z$-embedded null][end, no link to proof, text proof={Proof of Theorem \ref{thm:text0_implies_emb0_Z}}]\label{thm:text0_implies_emb0_Z}
    Let $\hat{\gamma}$ satisfy Assumption \ref{ass:gamma_no_cond_dep}, and
    \begin{align*}
        (i) \; X^{\omega_X,n}\indep Z^n\mid (\hat{\gamma}, \hat{\beta},Z^{\omega_Z,n})\quad \text{or} \quad (ii)\;  Y^n\indep Z^n\mid (\hat{\gamma}, \hat{\beta},Z^{\omega_Z,n}).
    \end{align*} 
    Then $H_0^{\omega_X}\implies H_0^{\omega}$.
\end{theoremE}
\begin{proofE}
By Lemma~\ref{lem:disintegration} applied to $\hat\beta$ and $S^{\omega_X}$, $H^{\omega_X}_0$ is equivalent to
\begin{equation}\label{eq:thm3-start}
  X^{\omega_X,n}\indep Y^n\mid(\hat\beta,Z^n).
\end{equation}

\emph{Step 1: adjoining $\hat\gamma$.} We show that \eqref{eq:thm3-start} together with
Assumption~\ref{ass:gamma_no_cond_dep} implies
\begin{equation}\label{eq:thm3-mid}
  X^{\omega_X,n}\indep Y^n\mid(\hat\beta,\hat\gamma,Z^n).
\end{equation}
Under Assumption~\ref{ass:gamma_no_cond_dep}a), contraction applied to \eqref{eq:thm3-start} and
$\hat\gamma\indep Y^n\mid((X^{\omega_X},Z)^n,\hat\beta)$ gives
$(X^{\omega_X,n},\hat\gamma)\indep Y^n\mid(\hat\beta,Z^n)$. Under
Assumption~\ref{ass:gamma_no_cond_dep}b), contraction applied to the symmetric form of \eqref{eq:thm3-start} and
$\hat\gamma\indep X^{\omega_X,n}\mid((Y,Z)^n,\hat\beta)$ gives
$(Y^n,\hat\gamma)\indep X^{\omega_X,n}\mid(\hat\beta,Z^n)$. In either case, weak union with the
measurable function $\hat\gamma$ of the left-hand pair, followed by decomposition, yields
\eqref{eq:thm3-mid}.

\emph{Step 2: replacing $Z^n$ by $Z^{\omega_Z,n}$.} Write $C_0=(\hat\beta,\hat\gamma,Z^{\omega_Z,n})$.
Since $Z^{\omega_Z,n}$ is $\sigma(Z^n,\hat\gamma)$-measurable, we have
$\sigma(C_0,Z^n)=\sigma(\hat\beta,\hat\gamma,Z^n)$, so \eqref{eq:thm3-mid} reads
\begin{equation}\label{eq:thm3-mid2}
  X^{\omega_X,n}\indep Y^n\mid(C_0,Z^n),
\end{equation}
which is symmetric in $X^{\omega_X,n}$ and $Y^n$. Under (i), contraction applied to
$Z^n\indep X^{\omega_X,n}\mid C_0$ and \eqref{eq:thm3-mid2} gives
$(Z^n,Y^n)\indep X^{\omega_X,n}\mid C_0$; under (ii), the same argument with the roles of
$X^{\omega_X,n}$ and $Y^n$ interchanged gives $(Z^n,X^{\omega_X,n})\indep Y^n\mid C_0$. In
either case decomposition yields $X^{\omega_X,n}\indep Y^n\mid C_0$, which by
Lemma~\ref{lem:disintegration} is $H^{\omega}_0$. All four combinations of
Assumption~\ref{ass:gamma_no_cond_dep}a)--b) with (i)--(ii) are therefore covered.
\end{proofE}
The conditions on $\hat\beta$, $\hat\gamma$ and $\omega_Z$ used above are statements about the joint law of $(S,\hat\beta,\hat\gamma)$ and hence properties of $\mathbb{P}^S$, so we may collect the null distributions satisfying them:
\begin{align*}
\mathcal{P}_0^{\mathrm{suff}} := \Big\{\mathbb{P}^S\in\mathcal{P}_0 :\
&\hat\beta\indep Y^n\mid (X,Z)^n,\ \text{Assumption~\ref{ass:gamma_no_cond_dep} holds, and}\\
&\text{(i) or (ii) of Theorem~\ref{thm:text0_implies_emb0_Z} holds}\Big\}.
\end{align*}
The class depends on the estimation procedure, $\omega_X$, $\omega_Z$ and $n$, which we suppress in the notation.

\begin{definition}[Validity]\label{def:validity}
Let $\alpha\in(0,1)$ and let $\mathcal{P}\subseteq\mathcal{P}_0$. A test $\varphi_\omega$ is
\emph{valid for $\mathcal{P}$ at level $\alpha$} if
\[
  \sup_{\mathbb{P}^S\in\mathcal{P}}\mathbb{P}^S\bigl(\varphi_\omega=1\bigr)\le\alpha .
\]
Analogously, for a class $\mathcal Q$ of probability measures on $(\mathcal S^\omega,\mathcal F_{\mathcal S^\omega})$, the test $\psi$ is valid for $\mathcal Q$ at level $\alpha$ if $\sup_{Q\in\mathcal Q}Q(\psi=1)\le\alpha$.
\end{definition}

\begin{theoremE}[Validity of the eCIT][end, no link to proof, text proof={Proof of Theorem \ref{thm:ecit-valid}}]\label{thm:ecit-valid}
Let $\alpha\in(0,1)$. If the CIT $\psi$ is valid for $\mathcal{P}_0^{\omega}$ at
level $\alpha$, then the eCIT $\varphi_{\omega}=\psi\circ\bar\omega$ is valid for
$\mathcal{P}_0^{\mathrm{suff}}$ at level $\alpha$.
\end{theoremE}
\begin{proofE}
Fix $\mathbb{P}^{S}\in\mathcal{P}_0^{\mathrm{suff}}$ and write $\hat\theta=(\hat\beta,\hat\gamma)$.

\emph{Step 1: $H_0\Rightarrow H_0^{\omega_X}$.} Since $\mathbb{P}^{S}\in\mathcal{P}_0$, the
$(X_i,Y_i,Z_i)$ are i.i.d.\ with $X\indep Y\mid Z$, so $X^n\indep Y^n\mid Z^n$ by
Lemma~\ref{lem:lift}. By definition of $\mathcal{P}_0^{\mathrm{suff}}$ we have
$\hat\beta\indep Y^n\mid (X,Z)^n$, and contraction gives $(X^n,\hat\beta)\indep Y^n\mid Z^n$.
By symmetry of conditional independence, $Y^n\indep(X^n,\hat\beta)\mid Z^n$; weak union with
the measurable function $\hat\beta$ of $(X^n,\hat\beta)$ yields
$Y^n\indep(X^n,\hat\beta)\mid(Z^n,\hat\beta)$, and decomposition, applied to the
$\sigma(X^n,\hat\beta)$-measurable $X^{\omega_X,n}$, yields
$X^{\omega_X,n}\indep Y^n\mid(Z^n,\hat\beta)$, which by Lemma~\ref{lem:disintegration} is
$H_0^{\omega_X}$.

\emph{Step 2: $H_0^{\omega_X}\Rightarrow H_0^{\omega}$.} By definition of
$\mathcal{P}_0^{\mathrm{suff}}$, $\hat\gamma$ satisfies Assumption~\ref{ass:gamma_no_cond_dep} and $\omega_Z$
satisfies (i) or (ii) of Theorem~\ref{thm:text0_implies_emb0_Z}, so Theorem~\ref{thm:text0_implies_emb0_Z} gives
$X^{\omega_X,n}\indep Y^n\mid(Z^{\omega_Z,n},\hat\theta)$. By Lemma~\ref{lem:disintegration} with
$V=S^{\omega}$ and $W=\hat\theta$,
\begin{equation}\label{eq:as-null}
\mathbb{P}^{S^{\omega}\mid\hat\theta}\in\mathcal{P}_0^{\omega}
\qquad\mathbb{P}^{\hat\theta}\text{-a.s.}
\end{equation}

\emph{Step 3: a valid CIT implies a valid eCIT.} Since $S^{\omega}=\bar\omega(S,\hat\theta)$ is
$\sigma(S,\hat\theta)$-measurable and $U\indep(S,\hat\theta)$, we have
$U\indep(S^{\omega},\hat\theta)$, so
$\mathbb{P}^{S^{\omega}\mid\hat\theta}\otimes\mathrm{Unif}[0,1]$ is a regular conditional
distribution of $(S^{\omega},U)$ given $\hat\theta$. The map
$(s^{\omega},u)\mapsto\mathbf{1}\{\psi(s^{\omega},u)=1\}$ is bounded and measurable,
so by the defining property of the regular conditional distribution,
\[
\mathbb{E}\bigl[\mathbf{1}\{\psi(S^{\omega},U)=1\}\,\big|\,\hat\theta\bigr]
=\int\mathbf{1}\{\psi=1\}\,
 \mathrm{d}\bigl(\mathbb{P}^{S^{\omega}\mid\hat\theta}\otimes\mathrm{Unif}[0,1]\bigr)
\qquad\mathbb{P}^{\hat\theta}\text{-a.s.}
\]
By \eqref{eq:as-null} the integrating measure is of the form $Q\otimes\mathrm{Unif}[0,1]$ with
$Q\in\mathcal{P}_0^{\omega}$ almost surely, so validity of $\psi$, being uniform over
$\mathcal{P}_0^{\omega}$, bounds the right-hand side by $\alpha$ almost surely. Since
$\varphi_{\omega}(S)=\psi(S^{\omega},U)$, the tower property gives
\[
\mathbb{P}^{S}\bigl(\varphi_{\omega}=1\bigr)
=\mathbb{E}\Bigl[\mathbb{E}\bigl[\mathbf{1}\{\psi(S^{\omega},U)=1\}\,\big|\,
\hat\theta\bigr]\Bigr]\le\alpha .
\]
As $\mathbb{P}^{S}\in\mathcal{P}_0^{\mathrm{suff}}$ was arbitrary, taking the supremum gives the
claim.
\end{proofE}

\begin{remark}\label{rem:no-marginal}
Unlike in \citet[Theorem 2]{simnacher2026deep}, the embedded null cannot be stated for the marginal law of
$S^\omega$. There, the assumption on $\hat\beta$ allows $\hat\beta$ to be removed from the conditioning
set by contraction and decomposition, because $X^{\omega_X,n}$ appears on the left of the conditioning
bar and $\hat\beta$ is a free element of the conditioning set. Here $Z^{\omega_Z,n}$ is
$\sigma(Z^n,\hat\gamma)$-measurable and appears on the right, so $\hat\gamma$ cannot be removed: deleting it
changes the $\sigma$-field relative to which the conditioning variable is defined.
Marginalizing instead averages over $\hat\gamma$, and a mixture of conditionally independent laws need not be conditionally independent. The almost-sure formulation
\eqref{equ:hypothesis_embedding} and the conditional argument in Step~3 of the proof of
Theorem~\ref{thm:ecit-valid} are therefore needed.
\end{remark}

Theorem~\ref{thm:ecit-valid} requires $\psi$ to be valid over all of $\mathcal{P}_0^{\omega}$, whose elements are not required to be product measures. Two regimes have to be distinguished. If $\hat\gamma$ is estimated on data independent of $S$ --- externally, or on a training split disjoint from the test sample --- then conditionally on $\hat\theta$ the rows of $S^\omega$ remain i.i.d., and validity over the i.i.d.\ null class suffices.

Let $\mathcal P_0^{\omega,\otimes}$ denote the $n$-fold products $M^{\otimes n}$ of probability measures $M$ on the embedded observation space under which $X^{\omega_X}\indep Y\mid Z^{\omega_Z}$, which by Lemma~\ref{lem:lift} implies $\mathcal P_0^{\omega,\otimes}\subseteq\mathcal P_0^\omega$.

\begin{corollaryE}[Independently estimated embedding parameters][end, no link to proof, text proof={Proof of Corollary \ref{cor:indep}}]\label{cor:indep}
Let $\alpha\in(0,1)$ and let $\psi$ be valid for $\mathcal{P}^{\omega,\otimes}_0$ at level $\alpha$. Then $\varphi_\omega$ is valid at level $\alpha$ for $\{\mathbb{P}^S\in\mathcal{P}^{\mathrm{suff}}_0:\hat\theta\indep S\}$.
\end{corollaryE}
\begin{proofE}
Fix such a $\mathbb{P}^S$. Since $\hat\theta\indep S$, the conditional law of $S$ given
$\hat\theta$ is $\mathbb{P}^S$ itself, and conditionally on $\hat\theta=(b,c)$ the rows of
$S^\omega$ are the images of the i.i.d.\ rows of $S$ under the fixed measurable map
$(x,y,z)\mapsto(\omega_X(x,b),y,\omega_Z(z,c))$. Hence $\mathbb{P}^{S^\omega\mid\hat\theta}$ is an
$n$-fold product of a common marginal $\mathbb{P}^{\hat\theta}$-a.s. Combined with Lemmas~\ref{lem:disintegration} and \ref{lem:lift} it lies in $\mathcal{P}^{\omega,\otimes}_0$ a.s., and Step~3 of the proof of
Theorem~\ref{thm:ecit-valid} applies with $\mathcal{P}^{\omega,\otimes}_0$ in place of
$\mathcal{P}^{\omega}_0$.
\end{proofE}

If instead $\hat\gamma$ is estimated in sample, Lemma~\ref{lem:lift} no longer applies at the embedded level: the conditional laws are in general not product measures, and $\psi$ must be valid over the strictly larger class $\mathcal{P}_0^{\omega}$. 
We therefore state Theorem~\ref{thm:text0_implies_emb0_Z} and Theorem~\ref{thm:ecit-valid} for general $\hat\gamma$, since the hypothesis-transfer argument requires no product structure, but restrict $\hat\gamma$ to embedding maps trained on a held-out corpus in the application in Subsection~\ref{subsec:real_world_application}, so that Corollary~\ref{cor:indep} applies and the conditional null class is the i.i.d.\ one.

While Theorem~\ref{thm:text0_implies_emb0_Z} shows that a sufficient embedding map $\omega_Z$ transfers the null hypothesis for texts $Z$ to the null hypothesis for embedded texts $Z^{\omega_Z}$, it does not establish equivalence of the two hypotheses. Thus, the embedded null \eqref{equ:hypothesis_embedding} could be true while the text null \eqref{equ:hypothesis_embedding_X} is false. For example, $Z^{\omega_Z}$ could contain $Y$, so that conditioning on $Z^{\omega_Z,n}$ makes $Y^n$ measurable with respect to the conditioning set and \eqref{equ:hypothesis_embedding} holds trivially. For the converse implication we have to exclude such cases, and we additionally need that $\hat\gamma$ may be removed from the conditioning set. 

\begin{assumption}\label{ass:gamma_converse}
Either a) $\hat\gamma\indep Y^n\mid (Z^n,\hat\beta)$ or
b) $\hat\gamma\indep X^{\omega_X,n}\mid (Z^n,\hat\beta)$.
\end{assumption}

Assumption~\ref{ass:gamma_converse} holds trivially whenever $\hat\gamma$ is learned on data independent of $S$, and whenever $\hat\gamma$ is a measurable function of $Z^n$ alone, for instance for unsupervised embedding maps of the source texts. 

\begin{theoremE}[Embedded null implies text null][end, no link to proof, text proof={Proof of Theorem \ref{thm:emb0_implies_text0_Z}}]\label{thm:emb0_implies_text0_Z}
Let
\begin{align*}
(i)\; X^{\omega_X,n}\indep Z^n\mid (\hat\beta,\hat\gamma,Z^{\omega_Z,n},Y^n) \quad\text{and Assumption~\ref{ass:gamma_converse}a)},\quad\text{or}\\
(ii)\; Y^n\indep Z^n\mid (\hat\beta,\hat\gamma,Z^{\omega_Z,n},X^{\omega_X,n}) \quad\text{and Assumption~\ref{ass:gamma_converse}b)}.
\end{align*}
Then $H_0^\omega\implies H_0^{\omega_X}$.
\end{theoremE}
\begin{proofE}
Assume (i) and write $C_0=(Z^{\omega_Z,n},\hat\beta,\hat\gamma)$. By Lemma~\ref{lem:disintegration}, $H_0^\omega$ is equivalent to $X^{\omega_X,n}\indep Y^n\mid C_0$. By (i), contraction 
gives
\[
X^{\omega_X,n}\indep (Y,Z)^n\mid C_0 .
\]
Since $Z^n$ is a measurable function of $(Y,Z)^n$, weak union yields $X^{\omega_X,n}\indep (Y,Z)^n\mid (C_0,Z^n)$, and decomposition yields $X^{\omega_X,n}\indep Y^n\mid (Z^n,Z^{\omega_Z,n},\hat\beta,\hat\gamma)$. As $Z^{\omega_Z,n}$ is $\sigma(Z^n,\hat\gamma)$-measurable, this is
\begin{equation}\label{eq:converse_intermediate}
X^{\omega_X,n}\indep Y^n\mid (Z^n,\hat\beta,\hat\gamma).
\end{equation}
Applying contraction 
gives $(\hat\gamma,X^{\omega_X,n})\indep Y^n\mid (Z^n,\hat\beta)$ by Assumption~\ref{ass:gamma_converse}a), and decomposition gives
\[
X^{\omega_X,n}\indep Y^n\mid (Z^n,\hat\beta),
\]
which by Lemma~\ref{lem:disintegration} is $H_0^{\omega_X}$. Under (ii) the same argument applies with the roles of $X^{\omega_X,n}$ and $Y^n$ interchanged, using the symmetry of conditional independence.
\end{proofE}

Finally, a joint CI assumption gives the equivalence:

\begin{corollaryE}[][end, no link to proof, text proof={Proof of Corollary \ref{cor:equivalence_nulls}}]\label{cor:equivalence_nulls}
Let Assumptions~\ref{ass:gamma_no_cond_dep} and \ref{ass:gamma_converse} hold and let
\[
(*):\quad (X^{\omega_X,n},Y^n)\indep Z^n\mid (\hat\beta,\hat\gamma,Z^{\omega_Z,n}).
\]
Then $H_0^\omega\iff H_0^{\omega_X}$.
\end{corollaryE}
\begin{proofE}
Write $C_0=(Z^{\omega_Z,n},\hat\beta,\hat\gamma)$. By symmetry of CI, $(*)$ reads $Z^n\indep (X^{\omega_X,n},Y^n)\mid C_0$, so decomposition gives both $X^{\omega_X,n}\indep Z^n\mid C_0$ and $Y^n\indep Z^n\mid C_0$, i.e.\ conditions (i) and (ii) of Theorem~\ref{thm:text0_implies_emb0_Z}. Moreover, weak union applied to $(*)$ with the measurable functions $Y^n$ and $X^{\omega_X,n}$ of $(X^{\omega_X,n},Y^n)$, each followed by decomposition, gives $X^{\omega_X,n}\indep Z^n\mid (C_0,Y^n)$ and $Y^n\indep Z^n\mid (C_0,X^{\omega_X,n})$, i.e.\ conditions (i) and (ii) of Theorem~\ref{thm:emb0_implies_text0_Z}. The claim now follows from Theorems~\ref{thm:text0_implies_emb0_Z} and \ref{thm:emb0_implies_text0_Z}.
\end{proofE}

Corollary~\ref{cor:equivalence_nulls} is a statement about the hypotheses, not about the power of a specific test: under $(*)$, the embedding map cannot map an alternative in $\mathcal{P}_1^{\omega_X}$ into the embedded null, so no alternative becomes undetectable in principle. It does not imply that the eCIT has the same power as a test of \eqref{equ:hypothesis_embedding_X}, since the embedding may reduce the effect size and project onto alternatives not detectable by the CIT applied to test \eqref{equ:hypothesis_embedding}.

Throughout this subsection we treated $\hat\beta$ as already learned, e.g.\ on an external dataset, and considered only the estimation of $\hat\gamma$. This weakens the requirement on $\omega_Z$. Since $X^{\omega_X,n}$ is $\sigma(X^n,\hat\beta)$-measurable, we have $\sigma(X^{\omega_X,n})\subseteq\sigma(X^n,\hat\beta)$, and condition~(i) of Theorem~\ref{thm:text0_implies_emb0_Z} requires only
\[
  X^{\omega_X,n}\;\indep\;Z^n \mid (\hat\beta,\hat\gamma,Z^{\omega_Z,n})
  \qquad\text{instead of}\qquad
  X^{n}\;\indep\;Z^n \mid (\hat\beta,\hat\gamma,Z^{\omega_Z,n}),
\]
so that $Z^{\omega_Z}$ need only retain the information in $Z$ about the representation $X^{\omega_X}$, and not about the whole text $X$. The coarser the embedding $\omega_X$, the weaker this condition becomes. 
However, if $\omega_X$ does not represent the conditional association between $X$ and $Y$ in a form the CIT can detect, the test loses power \citep{simnacher2026deep}.
As $X$ is embedded in any case to apply \citet{simnacher2026deep}, the choice is not whether to embed $X$ but in which order: embedding $X$ first weakens the sufficiency requirement on $\omega_Z$ without any additional loss of power, and may thereby facilitate T1E control.

\subsection{Conditional Mean Independence}\label{subsec:cond_mean_indep}

Throughout the paper, we focus on the projected covariance measure \citep[PCM;][]{lundborg2024projected} to test hypothesis \eqref{equ:hypothesis_embedding}, an algorithm-agnostic CMIT which performs well for low-dimensional, vector-valued $Z$ and embedded images $X$ \citep{simnacher2026deep}. However, other conditional (mean) independence tests \citep[e.g.,][]{strobl2019approximate, dai2022significance, williamson2023general, cai2025test} can be applied to test \eqref{equ:hypothesis_embedding}. 

Throughout this subsection we take $Y$ to be $\mathbb R^K$-valued, so that conditional means of $Y$ are defined: $K=1$ in the continuous case $\mathcal Y\subseteq\mathbb R$, and $K=L-1$ in the categorical case $\mathcal Y=\{1,\dots,L\}$, where $Y$ is represented by its vector of indicators of the non-baseline classes (cf.\ Remark~\ref{remark:equivalence_ci_mean_cat_Y} and
Appendix~\ref{app:catpcm}). We assume throughout that $\mathbb E\|Y_1\|<\infty$, where $\|\cdot\|$ is any norm on $\mathbb R^K$, all of which are equivalent, and which reduces to the absolute value for $K=1$. The assumption holds automatically in the categorical case and, since the rows of $S$ are i.i.d., gives $\mathbb E\|Y_i\|<\infty$ for all $i\le n$, so that the conditional means of $Y_i$ appearing below are well defined.

We test \eqref{equ:hypothesis_embedding} by applying the PCM, a CMIT, 
which tests instead
\begin{equation}\label{eq:mean-hyp}
  H_0^{\omega,\mathrm{mean}}:\
  \mathbb P^{S^\omega\mid\hat\theta}\in\mathcal P_0^{\omega,\mathrm{mean}}
  \ \ \mathbb P^{\hat\theta}\text{-a.s.}
  \quad\text{vs.}\quad
  H_1^{\omega,\mathrm{mean}}:\
  \mathbb {P}^{\hat\theta}\big(\mathbb P^{S^\omega\mid\hat\theta}
  \in\mathcal P_0^{\omega,\mathrm{mean}}\big)<1,
\end{equation}
where
\begin{align*}
  \mathcal P_0^{\omega,\mathrm{mean}}:=\Big\{&Q \text{ probability measure on } (\mathcal S^\omega,\mathcal F_{\mathcal S^\omega}) : \mathbb E_Q\|Y_i\|<\infty \text{ and }\\ 
  &\mathbb E_Q\big[Y_i\mid X^{\omega_X,n},Z^{\omega_Z,n}\big] =\mathbb E_Q\big[Y_i\mid Z^{\omega_Z,n}\big]\ Q\text{-a.s. for all } i\le n\Big\}.
\end{align*}
Since CI implies conditional mean independence and $Y_i$ is a measurable function of $Y^n$, every integrable $Q\in\mathcal P_0^\omega$ lies in $\mathcal P_0^{\omega,\mathrm{mean}}$. Integrability is not implied by $\mathcal{P}_0^{\omega}$ and has to be required separately, but it is not a restriction here: the embedding maps act on $X$ and $Z$ only, so the standing assumption $\mathbb{E}\|Y_1\|<\infty$ carries over to the embedded sample space and the conditional law $\mathbb{P}^{S^\omega\mid \hat\theta}$ is integrable almost surely. Whenever this conditional law lies in $\mathcal{P}_0^{\omega}$ it therefore lies in $\mathcal{P}_0^{\omega,mean}$, and since Step 3 of the proof of Theorem~\ref{thm:ecit-valid} uses only that it lies almost surely in the class over which $\psi$ is valid, Theorem~\ref{thm:ecit-valid} applies unchanged.
However, power against alternatives in $\mathcal P_0^{\omega,\mathrm{mean}}\setminus\mathcal P_0^\omega$ can be lost \citep{lundborg2024projected}, i.e.\ those in which the summary carries information about $Y$ beyond the speech without shifting its conditional mean.

As in the previous subsection, we embed $X$ first and characterize afterwards mean sufficiency for the embedding of $Z$. Thus, we define
\begin{equation}\label{eq:mean-embedding_X}
  H_0^{\omega_X,\mathrm{mean}}:\
  \mathbb P^{S^{\omega_X}\mid\hat\beta}\in\mathcal P_0^{\omega_X,\mathrm{mean}}
  \ \ \mathbb P^{\hat\beta}\text{-a.s.}
  \quad\text{vs.}\quad
  H_1^{\omega_X,\mathrm{mean}}:\
  \mathbb {P}^{\hat\beta} \big(\mathbb P^{S^{\omega_X}\mid\hat\beta}
  \in\mathcal P_0^{\omega_X,\mathrm{mean}}\big)<1,
\end{equation}
where 
\begin{align*}
    \mathcal P_0^{\omega_X,\mathrm{mean}}:=\Big\{&Q \text{ probability measure on } (\mathcal S^{\omega_X},\mathcal F_{\mathcal S^{\omega_X}}) : \mathbb E_Q\|Y_i\|<\infty \text{ and }\\ 
  &\mathbb E_Q\big[Y_i\mid X^{\omega_X,n},Z^{n}\big] =\mathbb E_Q\big[Y_i\mid Z^{n}\big]\ Q\text{-a.s. for all } i\le n\Big\}.
\end{align*}
Because \eqref{eq:mean-embedding_X} is stated only for the conditional mean, the assumption $\hat\beta\indep Y^n\mid (X,Z)^n$ under which \citet{simnacher2026deep} transfer $H_0$ to $H_0^{\omega_X}$ can be weakened. Specifically, $\hat\beta$ is only required to introduce no conditional mean dependence:
\begin{assumption}\label{ass:beta-mean}
$\mathbb E\big[Y_i\mid X^{n},Z^{n},\hat\beta\big]
 =\mathbb E\big[Y_i\mid X^{n},Z^{n}\big]$ a.s.\ for all $i\le n$.
\end{assumption}
Assumption~\ref{ass:beta-mean} is implied by $\hat\beta\indep Y^n\mid (X,Z)^n$ and holds in particular whenever $\hat\beta$ is a measurable function of $(X,Z)^n$ or is learned on data independent of $S$ \citep[][Corollary~3]{simnacher2026deep}. 
Next, we show that this assumption is sufficient to transfer the null hypothesis \eqref{equ:hypothesis_text} including $X$ and $Z$ to the null hypothesis $H_0^{\omega_X,\text{mean}}$ including only a representation of $X$.
\begin{theoremE}[Text null implies $X$-embedded mean null][end, no link to proof, text proof={Proof of Theorem \ref{thm:beta-mean}}]\label{thm:beta-mean}
Let Assumption~\ref{ass:beta-mean} hold. Then
$H_0\implies H_0^{\omega_X,\mathrm{mean}}$.
\end{theoremE}
\begin{proofE}
Write $\mu_i:=\mathbb E[Y_i\mid Z^n]$, which is well defined since
$\mathbb E\|Y_1\|<\infty$.

\emph{Step 1.} Under $H_0$ the rows of $S$ are i.i.d.\ with $X\indep Y\mid Z$, so
$X^n\indep Y^n\mid Z^n$ by Lemma~\ref{lem:lift}. Decomposition applied to the
measurable function $Y_i$ of $Y^n$ gives $X^n\indep Y_i\mid Z^n$, hence
\[
  \mathbb E\big[Y_i\mid X^n,Z^n\big]=\mu_i\quad\text{a.s. for all } i\le n .
\]

\emph{Step 2.} By Assumption~\ref{ass:beta-mean},
$\mathbb E[Y_i\mid X^n,Z^n,\hat\beta]=\mu_i$ a.s.

\emph{Step 3.} Since $X^{\omega_X,n}=\omega_X(X^n,\hat\beta)$ is
$\sigma(X^n,\hat\beta)$-measurable, we have
$\sigma(X^{\omega_X,n},Z^n,\hat\beta)\subseteq\sigma(X^n,Z^n,\hat\beta)$, so the tower
property and Step 2 give
\[
  \mathbb E\big[Y_i\mid X^{\omega_X,n},Z^n,\hat\beta\big]
  =\mathbb E\big[\mu_i\mid X^{\omega_X,n},Z^n,\hat\beta\big]=\mu_i ,
\]
the last equality because $\mu_i$ is $\sigma(Z^n)$-measurable and
$\sigma(Z^n)\subseteq\sigma(X^{\omega_X,n},Z^n,\hat\beta)$.

\emph{Step 4.} Since $\sigma(Z^n,\hat\beta)\subseteq\sigma(X^{\omega_X,n},Z^n,\hat\beta)$,
the tower property applied to Step 3 gives
$\mathbb E[Y_i\mid Z^n,\hat\beta]=\mathbb E[\mu_i\mid Z^n,\hat\beta]=\mu_i$.

Steps 3 and 4 give $\mathbb E[Y_i\mid X^{\omega_X,n},Z^n,\hat\beta]
=\mathbb E[Y_i\mid Z^n,\hat\beta]$ a.s.\ for all $i\le n$, which by
Lemma~\ref{lem:disint-mean} applied with $V=S^{\omega_X}$, $W=\hat\beta$ and
$(A,B,C)=(X^{\omega_X,n},Y^n,Z^n)$ is $H_0^{\omega_X,\mathrm{mean}}$.
\end{proofE}

In addition, we can weaken the sufficiency condition on the embedding of $Z$ if only the mean-level null has to be transferred. 
Write 
\[
  m_i:=\mathbb E\big[Y_i\mid Z^n,\hat\beta\big],\qquad i=1,\dots,n,
\]
for the regression of $Y_i$ on the source texts and the embedding parameter $\hat\beta$. When $\hat\beta$ is deterministic or learned on data independent of $S$, it is independent of $(Y_i,Z^n)$ and $m_i=\mathbb E[Y_i\mid Z^n]$ is the regression of $Y_i$ on the source texts alone; $\hat\beta$ is retained in general because an in-sample $\hat\beta$ is a function of $(X,Z)^n$ and can then carry information about $Y_i$ beyond $Z^n$. 
We first state a mean-level version of Assumption~\ref{ass:gamma_no_cond_dep}a), which restricts $\hat\gamma$ to introduce no conditional mean dependence and holds in particular whenever $\hat\gamma$ is learned independently of $S$ or is a measurable function of $(X^{\omega_X},Z)^n$. 
We then characterize when $H_0^{\omega_X,\mathrm{mean}}$ transfers to $H_0^{\omega,\mathrm{mean}}$, and show that it does so whenever $m_i$ is $\sigma(Z^{\omega_Z,n},\hat\theta)$-measurable, that is, whenever $Z$ carries no information about the conditional mean of $Y$ beyond $Z^{\omega_Z}$. 
This replaces the requirement that $Z^{\omega_Z}$ preserve the entire conditional law of $Y$ or $X^{\omega_X}$ given $Z$ by a requirement on a single regression function.

\begin{assumption}\label{ass:mean-gamma}
$\mathbb E[Y_i\mid X^{\omega_X,n},Z^n,\hat\theta]
=\mathbb E[Y_i\mid X^{\omega_X,n},Z^n,\hat\beta]$ a.s.\ for all $i\le n$.
\end{assumption}

Assumption~\ref{ass:gamma_no_cond_dep}a) implies Assumption~\ref{ass:mean-gamma}, which restricts $\hat\gamma$ in the same way as Assumption~\ref{ass:beta-mean} restricts $\hat\beta$. There is no analogue of Assumption~\ref{ass:gamma_no_cond_dep}b), because conditional mean independence is not symmetric in its two arguments.

\begin{theoremE}[$Z$-text mean null implies $Z$-embedded mean null][end, no link to proof, text proof={Proof of Theorem \ref{thm:mean-transfer}}]\label{thm:mean-transfer}
Let Assumption~\ref{ass:mean-gamma} hold, and let
\[
  \text{(ii$'$)}\qquad
  m_i=\mathbb E\big[m_i\mid Z^{\omega_Z,n},\hat\theta\big]
  \quad\text{a.s. for all } i\le n ,
\]
that is, let each $m_i$ be almost surely equal to a
$\sigma(Z^{\omega_Z,n},\hat\theta)$-measurable random variable. Then
$H_0^{\omega_X,\mathrm{mean}}\implies H_0^{\omega,\mathrm{mean}}$.
\end{theoremE}
\begin{proofE}
\emph{Step 1: restating the hypotheses.} By Lemma~\ref{lem:disint-mean} applied with
$V=S^{\omega_X}$, $W=\hat\beta$ and $(A,B,C)=(X^{\omega_X,n},Y^n,Z^n)$, the hypothesis
$H_0^{\omega_X,\mathrm{mean}}$ is equivalent to
\begin{equation}\label{eq:mean-transfer-null}
  \mathbb E\big[Y_i\mid X^{\omega_X,n},Z^n,\hat\beta\big]
  =\mathbb E\big[Y_i\mid Z^n,\hat\beta\big]=m_i
  \quad\text{a.s. for all } i\le n ,
\end{equation}
and by the same lemma with $V=S^{\omega}$, $W=\hat\theta$ and
$(A,B,C)=(X^{\omega_X,n},Y^n,Z^{\omega_Z,n})$, the hypothesis $H_0^{\omega,\mathrm{mean}}$
is equivalent to
\begin{equation}\label{eq:mean-transfer-emb}
  \mathbb E\big[Y_i\mid X^{\omega_X,n},Z^{\omega_Z,n},\hat\theta\big]
  =\mathbb E\big[Y_i\mid Z^{\omega_Z,n},\hat\theta\big]
  \quad\text{a.s. for all } i\le n .
\end{equation}
All conditional expectations are well defined, since the rows of $S$ are i.i.d.\ and
$\mathbb E\|Y_1\|<\infty$.

\emph{Step 2: adjoining $\hat\gamma$.} Assumption~\ref{ass:mean-gamma} and
\eqref{eq:mean-transfer-null} give
\[
  \mathbb E\big[Y_i\mid X^{\omega_X,n},Z^n,\hat\theta\big]
  =\mathbb E\big[Y_i\mid X^{\omega_X,n},Z^n,\hat\beta\big]=m_i .
\]
Since $\sigma(Z^n,\hat\theta)\subseteq\sigma(X^{\omega_X,n},Z^n,\hat\theta)$ and $m_i$ is
$\sigma(Z^n,\hat\beta)$-measurable, hence $\sigma(Z^n,\hat\theta)$-measurable, the tower
property yields $\mathbb E[Y_i\mid Z^n,\hat\theta]=\mathbb E[m_i\mid Z^n,\hat\theta]=m_i$.
Neither side of the text-level mean null is therefore changed by adjoining $\hat\gamma$
to the conditioning set.

\emph{Step 3: replacing $Z^n$ by $Z^{\omega_Z,n}$.} As $Z^{\omega_Z,n}$ is
$\sigma(Z^n,\hat\gamma)$-measurable, we have
$\sigma(X^{\omega_X,n},Z^{\omega_Z,n},\hat\theta)\subseteq
\sigma(X^{\omega_X,n},Z^n,\hat\theta)$ and
$\sigma(Z^{\omega_Z,n},\hat\theta)\subseteq\sigma(Z^n,\hat\theta)$. Applying the tower
property to the two identities of Step~2 across these inclusions gives
\begin{equation}\label{eq:mean-transfer-step3}
  \mathbb E\big[Y_i\mid X^{\omega_X,n},Z^{\omega_Z,n},\hat\theta\big]
  =\mathbb E\big[m_i\mid X^{\omega_X,n},Z^{\omega_Z,n},\hat\theta\big],
  \qquad
  \mathbb E\big[Y_i\mid Z^{\omega_Z,n},\hat\theta\big]
  =\mathbb E\big[m_i\mid Z^{\omega_Z,n},\hat\theta\big].
\end{equation}
Under (ii$'$) both right-hand sides equal $m_i$ a.s., so the two left-hand sides agree for
all $i$, which is \eqref{eq:mean-transfer-emb} and hence
$H_0^{\omega,\mathrm{mean}}$.
\end{proofE}

The proof of Theorem~\ref{thm:mean-transfer} shows that $\mathbb E[Y_i\mid Z^{\omega_Z,n},\hat\theta]=\mathbb E[m_i\mid Z^{\omega_Z,n},\hat\theta]$, so (ii$'$) holds if and only if $m_i=\mathbb E[Y_i\mid Z^{\omega_Z,n},\hat\theta]$ a.s.\ for all $i\le n$, that is, if and only if the regression of $Y_i$ on the embedded source texts coincides with its regression on the source texts. 
Mean sufficiency of $\omega_Z$ is thus a condition on the regression of $Y$ on $Z$ alone and can be assessed without reference to $X^{\omega_X}$.
Condition (ii$'$) is an exact measurability requirement and is not plausible for embedding maps learned from data, which approximate e.g. the regression of $Y$ on the source texts rather than reproducing it. 
What matters in practice is whether the residual $m_i-\mathbb E[m_i\mid Z^{\omega_Z,n},\hat\theta]$ is negligible relative to $n^{-1/2}$, since a fixed nonzero residual enters the test statistic as a location shift growing with $\sqrt n$. 
It is nevertheless approachable, because $m_i$ need not depend on all of $Z_i$: if $m_i=g(t(Z_i))$ for a map $t$ and a function $g$, then (ii$'$) holds for every $\omega_Z$ for which $t(Z_i)$ is $ \sigma(Z^{\omega_Z,n},\hat\theta)$-measurable. 
In the simulation study of Subsection~\ref{subsec:simulation}, $t$ is a linear score in the embedding used to generate $Y$, so every eCIT whose $Z^{\omega_Z}$ retains that embedding satisfies (ii$'$) exactly; in the application in Subsection~\ref{subsec:real_world_application}, $t$ is unknown and we approximate it by a learned representation of the speech, the masked mean of the last hidden layer of a model predicting $Y$ from $Z$. Appending further embedding maps enlarges $\sigma(Z^{\omega_Z,n})$ and can therefore only reduce the residual, which motivates the concatenated conditioning sets used in both. Whether the residual is negligible at a given sample size depends on the dataset and the embedding maps, and we propose to assess it with the simulation design of Subsection~\ref{subsec:simulation}.

The null class $\mathcal P_0^{\omega,\mathrm{mean}}$ constrains the conditional mean of $Y_i$ given the whole embedded sample, whereas the PCM
fits regressions of $Y_i$ only on $(X^{\omega_X}_i,Z^{\omega_Z}_i)$ and on $Z^{\omega_Z}_i$. 
The two agree whenever the rows of $S^\omega$ are i.i.d.\ conditionally on $\hat\theta$, which is e.g.\ the case when the embedding parameters are estimated independently of $S$, as in our application. 
Let $\mathcal P_0^{\omega,\mathrm{mean},\otimes}$ denote the $n$-fold products $M^{\otimes n}$ of probability measures $M$ on the embedded observation space $(\mathfrak X\times\mathcal Y\times\mathfrak Z)$ under which $\mathbb E_M\|Y\|<\infty$ and $\mathbb E_M[Y\mid X^{\omega_X},Z^{\omega_Z}]=\mathbb E_M[Y\mid Z^{\omega_Z}]$, which by Lemma~\ref{lem:lifting-mean} satisfies $\mathcal P_0^{\omega,\mathrm{mean},\otimes}\subseteq\mathcal P_0^{\omega,\mathrm{mean}}$. Then, we collect the null distributions satisfying the conditions used above in
\begin{align*}
    \mathcal P_0^{\mathrm{suff,mean}}:=\Big\{\mathbb P^{S}\in\mathcal P_0:\ \text{Assumptions~\ref{ass:beta-mean} and \ref{ass:mean-gamma} hold, and (ii$'$) of Theorem~\ref{thm:mean-transfer} holds}\Big\},
\end{align*}
which depends on the estimation procedure, $\omega_X$, $\omega_Z$ and $n$, suppressed in
the notation as for $\mathcal P_0^{\mathrm{suff}}$. 
Step~3 of the proof of Theorem~\ref{thm:ecit-valid} uses only that $\mathbb P^{S^\omega\mid\hat\theta}$ lies almost surely in the class over which $\psi$ is valid, so a $\psi$ valid for $\mathcal P_0^{\omega,\mathrm{mean}}$ yields an eCIT valid for $\mathcal P_0^{\mathrm{suff,mean}}$.
Then, we can derive the analog of Corollary~\ref{cor:indep} for the conditional mean independence hypothesis.

\begin{corollaryE}[Independently estimated embedding parameters, mean version][end, no link to proof, text proof={Proof of Corollary~\ref{cor:mean-iid}}]
\label{cor:mean-iid}
Let $\alpha\in(0,1)$ and let $\psi$ be valid for $\mathcal P_0^{\omega,\mathrm{mean},\otimes}$ at level $\alpha$. Then $\varphi_\omega$ is valid at level $\alpha$ for $\{\mathbb P^{S}\in\mathcal P_0^{\mathrm{suff,mean}}:\hat\theta\indep S\}$.
\end{corollaryE}

\begin{proofE}
Fix such a $\mathbb P^{S}$. Since $\mathbb P^{S}\in\mathcal P_0$ and
Assumption~\ref{ass:beta-mean} holds, Theorem~\ref{thm:beta-mean} gives
$H_0^{\omega_X,\mathrm{mean}}$; since Assumption~\ref{ass:mean-gamma} and (ii$'$) hold,
Theorem~\ref{thm:mean-transfer} gives $H_0^{\omega,\mathrm{mean}}$, that is,
$\mathbb P^{S^\omega\mid\hat\theta}\in\mathcal P_0^{\omega,\mathrm{mean}}$
$\mathbb P^{\hat\theta}$-a.s.

Since $\hat\theta\indep S$, the conditional law of $S$ given $\hat\theta$ is
$\mathbb P^{S}$ itself, and conditionally on $\hat\theta=(b,c)$ the rows of $S^\omega$ are
the images of the i.i.d.\ rows of $S$ under the fixed measurable map
$(x,y,z)\mapsto(\omega_X(x,b),y,\omega_Z(z,c))$. Hence
$\mathbb P^{S^\omega\mid\hat\theta}$ is an $n$-fold product $M^{\otimes n}$ of a common
marginal $\mathbb P^{\hat\theta}$-a.s., and Lemma~\ref{lem:lifting-mean} applied to
$M^{\otimes n}$ gives $M^{\otimes n}\in\mathcal P_0^{\omega,\mathrm{mean},\otimes}$ a.s.

Step~3 of the proof of Theorem~\ref{thm:ecit-valid} applies verbatim with
$\mathcal P_0^{\omega,\mathrm{mean},\otimes}$ in place of $\mathcal P_0^\omega$, since it
uses only that $S^\omega$ is $\sigma(S,\hat\theta)$-measurable, that
$U\indep(S,\hat\theta)$, and that $\psi$ is valid uniformly over a class containing
$\mathbb P^{S^\omega\mid\hat\theta}$ almost surely.
\end{proofE}

\begin{remark}[Categorical $Y$]\label{remark:equivalence_ci_mean_cat_Y}
For $\mathcal Y$ finite and $Y$ represented by its vector of category indicators,
$\mathbb E_Q[Y_i\mid X^{\omega_X,n},Z^{\omega_Z,n}]=\mathbb E_Q[Y_i\mid Z^{\omega_Z,n}]$ determines the
conditional law of $Y_i$ and is therefore equivalent to $X^{\omega_X,n}\indep Y_i\mid Z^{\omega_Z,n}$,
for each $i\le n$. For product measures these row-wise statements lift to the joint one by
Lemma~\ref{lem:lift}, so $\mathcal P_0^{\omega,\mathrm{mean},\otimes}=\mathcal P_0^{\omega,\otimes}$,
and (ii$'$) is the row-wise counterpart of condition (ii), the two coinciding for product measures for the same reason.
For general $Q$ the mean class is the larger one, since row-wise conditional independence need not
imply $X^{\omega_X,n}\indep Y^n\mid Z^{\omega_Z,n}$. In our application
$\hat\theta\indep S$, so Corollary~\ref{cor:mean-iid} applies with the product class and nothing is
lost by using a CMIT.
\end{remark}

Finally, we note that the PCM guarantees only asymptotic uniform validity, so the eCIT is asymptotically valid for mean sufficient embedding maps. 
In our application the embedding parameters are estimated independently of $S$, so Corollary~\ref{cor:mean-iid} is the relevant statement and the null class is $\mathcal P_0^{\omega,\mathrm{mean},\otimes}$, whose elements are the i.i.d.\ laws over which the PCM is asymptotically uniformly valid. 
The conditional argument in Step~3 of the proof of Theorem~\ref{thm:ecit-valid} transfers uniform asymptotic validity, since the bound holds simultaneously for all $Q\in\mathcal P_0^{\omega,\mathrm{mean},\otimes}$ and hence for the realized conditional law $\mathbb P^{S^\omega\mid\hat\theta}$; pointwise asymptotic validity would not suffice, as the conditioning value $\hat\theta$ varies with $n$. For $\hat\gamma$ estimated in sample the conditional laws need not be products, and validity over the larger class $\mathcal P_0^{\omega,\mathrm{mean}}$ would be required, which the PCM is not known to provide.

\section{Application}\label{sec:application}

We apply the eCITs to German Parliament speeches and their summaries. 
First, we provide a description of the data in Subsection \ref{subsec:data}. 
Then, we evaluate the performance of the eCITs on semi-synthetic data based on the speech-summary pairs in Subsection~\ref{subsec:simulation}. 
While our simulation study evaluates the T1E control and power of the eCITs on this specific dataset and for the LLM summaries, the simulation design can be used to evaluate if the eCITs can control T1E and have power against alternatives of interest for other datasets with $Z$, $X$ pairs before applying the eCITs. 
On other datasets and their underlying data generating processes, the embedding maps used, their collinearity within and between $Z$ and $X$, and the corresponding separation for categorical $Y$ can affect the performance of the eCITs and particularly the T1E control with respect to the sufficiency of $Z^{\omega_Z}$ and the convergence of the nuisance regressions required by the PCM or other CITs. 
Finally, we apply in Subsection~\ref{subsec:real_world_application} the eCITs to the full dataset, including speech-summary pairs together with the protected attributes of the speakers' gender and faction. 

\subsection{Data}\label{subsec:data}

For the source texts, we use the German Parliament speeches provided in
\citet{richter2020opendiscourse}.
We clean and merge speeches as described in Appendix \ref{app:summary_gen_details}.
The resulting corpus comprises $274{,}598$ merged speeches from electoral terms~$1$--$20$, dated 12~September~1949--20~May~2022.
Roles are dominated by Member of Parliament (MP; $185{,}009$) and Presidium of Parliament ($34{,}392$), plus government speakers (Secretary of State $28{,}547$; Minister $19{,}892$; Chancellor $1{,}559$), guests ($5{,}184$), and 15 speeches with unlabeled role.
Merged speech length has median~$382$ words (mean~$605$; IQR~$122$--$794$), and there are $4{,}030$ unique speaker identifiers.

For the analysis, we further restrict this corpus to the $274{,}579$ speeches for which a Gemma 4 summary was generated, with 19 dropped speeches because they did not produce a summary.
Among merged MP speeches, the main parliamentary groups are CDU/CSU ($59{,}041$), SPD ($54{,}752$), FDP ($27{,}677$), Gr\"une ($21{,}369$), DIE LINKE.\ ($10{,}006$), PDS ($4{,}338$), and AfD ($2{,}742$), plus smaller and historical groups (KPD, DP, GB/BHE, Zentrum, and others) and speakers without faction (Fraktionslos).
We drop unlabeled factions and keep only factions at least as frequent as AfD, yielding seven factions (CDU/CSU, SPD, FDP, Gr\"une, DIE LINKE., PDS, AfD; $n=179{,}925$). 
For the analysis including the speakers' faction, we randomly draw a held-out test set of size $n_{\mathrm{test}}=53{,}978$ for the eCITs while the remainder ($n_{\mathrm{train}}=125{,}947$) is used to train embedding maps.
For the second attribute we use the speaker's gender as recorded in the corpus, which takes two values $(L=2)$.
Here, we keep electoral terms~$10$--$20$ with non-missing speaker gender ($n=169{,}463$; $n_{\mathrm{train}}=118{,}624$, $n_{\mathrm{test}}=50{,}839$), of which around $71\%$ are by male and $29\%$ by female speakers.
For the semi-synthetic simulation study, we use the intersection of the faction and gender train/test splits ($n_{\mathrm{train}}=61{,}807$ to train embeddings, and $n_{\mathrm{test}}=11{,}357$ to apply the eCITs).

\subsection{Simulation}\label{subsec:simulation}

We examine the empirical performance of the eCIT on simulated data. To evaluate whether the eCIT controls the T1E and has power against relevant alternatives on the dataset of interest, we resample $Z,X$ pairs and generate $Y$ from a pre-specified embedding of $Z$ (under the null hypothesis) or from prespecified embeddings of $Z$ and $X$ (under the alternative hypothesis).
Since we apply the PCM, a CMIT, and the embedding maps are trained on a holdout set disjoint from the test sample, so that $\hat\theta\indep S$, the relevant statements are Theorem~\ref{thm:mean-transfer} and Corollary~\ref{cor:mean-iid}, and the condition on $\omega_Z$ is mean sufficiency, i.e.\ (ii$'$). For the binary and multinomial outcomes this coincides with sufficiency by Remark~\ref{remark:equivalence_ci_mean_cat_Y}, and we write sufficient for mean sufficient below.

\subsubsection{Design}

We compare the proposed eCITs across several combinations of fixed and learned embedding maps for $X$ and $Z$. 
Summaries are generated with Gemma 4 \citep{gemmateam2026gemma4} throughout, and additionally with Qwen 3.6 \citep{qwen3.6-27b} for the application in Subsection~\ref{subsec:real_world_application} (Appendix~\ref{app:summary_gen_details}).

\textbf{Data generating mechanisms (DGMs):} We consider 12 DGMs across three outcome types, two generating embedders, under the null and alternative hypothesis. For all, we vary the sample size over $n=1{,}000, 5{,}000, 10{,}000$. 
These DGMs extend the simulation studies of \citet{simnacher2026deep}, which evaluate CITs for low-dimensional $Z$ and images $X$, to high-dimensional texts $Z$ and $X$.

For the source texts $Z$, we resample repeatedly $n$ speeches from the test set (cf. Subsection~\ref{subsec:data}). 
We obtain the corresponding summaries $X$ using the Gemma 4 LLM.
This provides speech-summary pairs $(Z_i,X_i)$ as used in the real-world application. 
We embed each $Z_i$ and $X_i$ with the Jina embedder for text classification \citep{akram2026jinaembeddingsv5texttasktargetedembeddingdistillation} and the Qwen embedder for general tasks \citep{qwen3embedding}, truncated with their Matryoshka representations \citep{kusupati2022matryoshka} to $d = 32$ dimensions. Depending on the DGM, we generate $Y$ from one of the two embedders, whose maps for the speech and the summary we denote by $\widetilde\omega_Z$ and $\widetilde\omega_X$, i.e. $\widetilde\omega_Z, \widetilde\omega_X\in \{\text{Jina, Qwen} \}$, so that $z_i^{\widetilde\omega_Z}, x_i^{\widetilde\omega_X} \in \mathbb{R}^{32}$.

We draw two directions $B_Z, B_X \sim N(0_{32}, I_{32})$ once and compute the unnormalized scores
\[
  \tilde s_{Z,i} = \bigl(z_i^{\widetilde\omega_Z}\bigr)^\top B_Z,
  \qquad
  \tilde s_{X,i} = \bigl(x_i^{\widetilde\omega_X} - A^\top z_i^{\widetilde\omega_Z}\bigr)^\top B_X,
\]
where $A$ is the coefficient matrix of the regression of $x^{\widetilde\omega_X}$ on
$z^{\widetilde\omega_Z}$, fitted on the training split. Then, the speech and summary scores used in the DGMs are $s_{Z,i} = \tilde s_{Z,i} / \widehat{\mathrm{sd}}(\tilde s_Z)$ and $s_{X,i} = \tilde s_{X,i} / \widehat{\mathrm{sd}}(\tilde s_X)$, so that both have unit variance. The column standardization of the embeddings and the two standard deviations are likewise computed on the training split, so that all of these are held fixed across replications.
Thus $s_Z$ is a one-dimensional projection of the speech representation $z^{\widetilde\omega_Z}$, and $s_X$ a one-dimensional projection of the part of the summary representation $x^{\widetilde\omega_X}$ that the speech representation does not explain. 
Given $s_Z$ and $s_X$, we generate $Y$ under $c = 0$ (CI) and $c = 1$ (no CI) according to the outcome type:
\begin{align}
\text{continuous:}\quad & Y_i = s_{Z,i} + c\,s_{X,i} + \varepsilon_i, \quad
  \varepsilon_i \sim N(0, \sigma^2), \notag \\
\text{binary:}\quad & Y_i \sim \mathrm{Bernoulli}\bigl(\mathrm{expit}(\eta_i)\bigr), \quad
  \eta_i = \alpha_Z s_{Z,i} + c\,\alpha_X s_{X,i} + b, \label{eq:dgm} \\
\text{multinomial:}\quad & Y_i \sim \mathrm{Categorical}\bigl(\mathrm{softmax}(\eta_i)\bigr),
  \quad Y_i \in \{0, \dots, L-1\}, \notag
\end{align}
where in the multinomial case the reference class has $\eta_{i,0} = 0$ and the remaining $L-1$ logits are fixed linear combinations of $s_{Z,i}$ and $c\,s_{X,i}$, chosen so that the two scores load on the classes in different directions, with the softmax taken over all $L$ logits. 
We set $L = 5$. The coefficients $\alpha_Z, \alpha_X$, the class weights, the intercept $b$, and the noise variance $\sigma^2$ of the continuous outcome are given in Appendix~\ref{app:dgm_details}.
Under the null, $Y$ depends on the speech only through $s_{Z,i}$, and hence only through
$z^{\widetilde\omega_Z}$.

\textbf{Embeddings in eCITs:} First, we apply the eCITs with only the embedding maps used to generate $Y$ (oracle). 
Next, we apply the eCITs with additional embedding maps, including the BGE embedder \citep{chen2024m3}, the Llama Nemotron embedder \citep{moreira2024nv, nvidia_llama_nemotron_embed_1b_v2}, the Jina or Qwen embedder, depending on which was not used in the corresponding $Y$ generation, and learned embeddings. 
We refer to the eCITs combining the generating embedding with further embedding maps as diluted, and contrast them with those omitting the generating embedding.
We selected the four pre-trained embedding maps because they cover embedding models based on different architectures, they allow for the long context of the speeches, and they performed best in a recent meta study of embedding maps for text across several tasks \citep{gjorgjevikj2026robustness}. 
For the trained embedding maps, we predict $Y$ separately from the speech $Z$ and from the summary $X$ using a text prediction model with 134 million parameters  and classification or regression head \citep{wunderle2025new} on a large holdout set of $61{,}807$ speech-summary pairs together with the generated $Y$ in the corresponding DGM.

\textbf{Target and performance measures:}
Performance is measured in terms of the T1E and power. 
T1E and power are estimated from the rejection rates 
\begin{align*}
    \widehat{RR}=\frac{1}{n_{\text{sim}}}\sum_{l=1}^{n_{\text{sim}}}\mathds{1}\{p_l\leq \alpha\}
\end{align*}
under the null ($c=0$) and alternative ($c=1$) hypothesis, respectively, for a significance level of $\alpha=0.05$, over $l=1,\hdots,n_{\text{sim}}$ random data generations and using p-values $p_l$ for each DGM. We choose $n_{\text{sim}}=50$ for computational reasons, resulting at level $\alpha=0.05$ in a 
Monte Carlo (MC) standard error \citep[sec.~5.3]{morris2019using} of
\begin{align*}
    \text{MC SE}(\widehat{RR})=\sqrt{\frac{\widehat{RR}(1-\widehat{RR})}{n_{\text{sim}}}}\approx \sqrt{\frac{0.05^*0.95}{50}}\approx 0.031.
\end{align*}

\subsubsection{Results}

Figure~\ref{fig:type1_power_oracle_dilute_concat_jina_binary_multinomial} shows the rejection rates for binary and multinomial $Y$ generated from the Jina embeddings. 
The oracle eCIT, which uses the embedding maps $\widetilde\omega_Z, \widetilde\omega_X$ that generate $Y$, satisfies condition~(ii$'$) of Theorem~\ref{thm:mean-transfer} by construction, since under the null $m_i$ depends on $Z_i$ only through $s_{Z,i}$, which is a linear function of $z_i^{\mathrm{Jina}}$.
It therefore isolates the behavior of the PCM on this DGM, that is, whether the test controls the T1E and has power for the embedded hypothesis, and the results show that it holds the nominal level and rejects in every replication under the alternative already at $n = 1{,}000$. 
It thus serves as the reference for the remaining eCITs: deviations from it are attributable to the embedding maps and the PCM applied to the larger dimension of the corresponding combined embeddings.

\begin{figure}[t!]
    \centering
    \includegraphics[width=\linewidth]{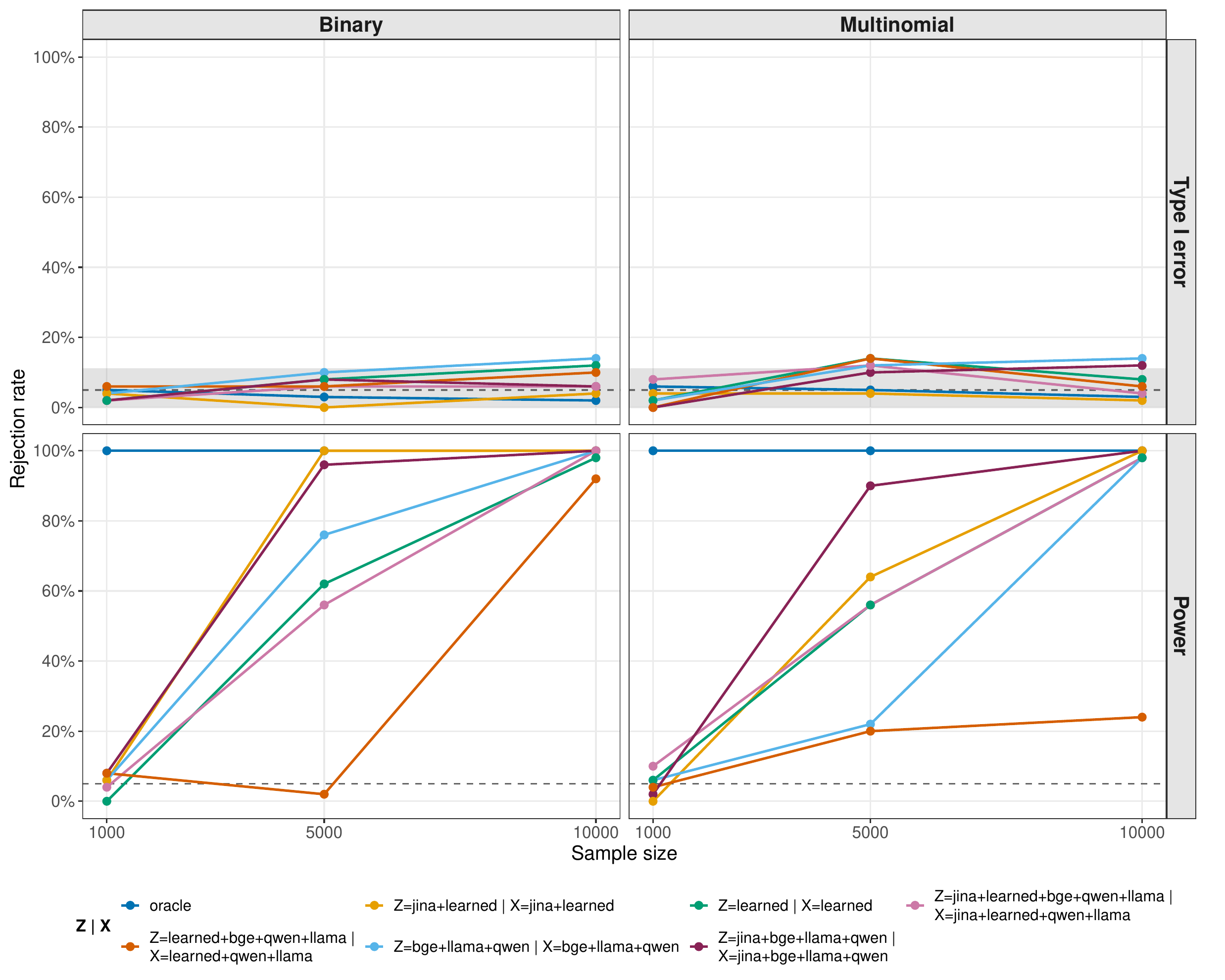}
    \caption{The rejection rates of the eCIT for binary and multinomial $Y$ generated under the null (first row) and alternative (second row) from the Jina embeddings $Z^{Jina}$ and $X^{Jina}$. The eCIT is applied with the Jina embeddings (oracle), the specified diluted embeddings, including additionally learned, BGE, Qwen, and Llama embeddings, as well as the same combinations excluding the Jina embedding. For the PCM, the nuisance regressions on the training dataset are $L_2$-penalized and those on the evaluation set unpenalized; for categorical $Y$ the regressions of $Y$ are multinomial logistic and those of the fitted probabilities and of $\hat f$ are least squares (Appendix~\ref{app:catpcm:test}).  The dashed line marks the nominal level $\alpha=0.05$, the shaded area marking two MC standard errors around it. The continuous outcome and the corresponding QQ-plots are shown in Figures~\ref{fig:type1_power_oracle_dilute_concat_jina_binary_multinomial_continuous} and~\ref{fig:qq_oracle_dilute_concat_jina_binary_multinomial_continuous}.}
    \label{fig:type1_power_oracle_dilute_concat_jina_binary_multinomial}
\end{figure}

Under the null and the Jina DGM, all eCITs stay within or close to two MC standard errors of $\alpha = 0.05$ for binary and multinomial $Y$, with a maximum rejection rate of around $13\%$ at $n = 10{,}000$.  
This holds also for the eCITs that omit the generating embedding $Z^{Jina}$ entirely, i.e.\ for conditioning sets for which sufficiency does not hold by construction, although for some of them, e.g.\ $Z^{\omega_Z} =\text{BGE}+\text{Llama}+\text{Qwen}$ and $Z^{\omega_Z} = \text{learned}$, the rejection rate increases mildly over the sample sizes considered.
The QQ-plots in Figure~\ref{fig:qq_oracle_dilute_concat_jina_binary_multinomial_continuous} show the same picture: for the categorical outcomes the p-values track the uniform diagonal closely at $n = 1{,}000$, with departures for individual eCITs at $n = 10{,}000$ that remain small.
In general, level control does not require $Z^{\omega_Z}$ to contain the generating embedding, only that it retain enough of $s_Z$ for the transfer of the null hypothesis on the level of the texts to the one on the level of the embedded texts (Theorem~\ref{thm:mean-transfer}) to hold approximately. Here, the part of $s_Z$ that is lost by omitting $Z^{\text{Jina}}$ is still negligible at $\sqrt{n}$ scale at the sample sizes considered.

Under the alternative and the Jina DGM, the combination of the embeddings of $X^{\omega_X}$ and $Z^{\omega_Z}$ affects the power of the corresponding eCITs. 
The ordering of the eCITs follows from how much of $s_X$ is retained in $X^{\omega_X}$.
At $n = 1{,}000$ no non-oracle arm exceeds around $10\%$. At $n = 5{,}000$ the arms that contain the generating embeddings $Z^{Jina}, X^{Jina}$ reach $55$--$97\%$ for binary and $56$--$90\%$ for multinomial $Y$, whereas the eCITs omitting the generating embeddings have lower or similar power. 
The eCIT with $Z^{\omega_Z} = \text{learned}+\text{BGE}+\text{Qwen}+\text{Llama}$ and $X^{\omega_X} = \text{learned}+\text{Qwen}+\text{Llama}$ achieves the lowest power for binary and multinomial $Y$. Here, the BGE embedding in $Z^{\omega_Z}$ is not included in $X^{\omega_X}$, lowering the detectable signal.
Adding embeddings to $X^{\omega_X}$  can thus recover power, depending on how well the available maps span the conditional dependence of the DGM.
Finally, eCITs for multinomial $Y$ are less powerful than for binary $Y$ at the largest sample size. 

The Qwen DGM separates including and omitting the generating embedding $Z^{Qwen}$, and the corresponding (in)sufficiency of the embedding of $Z$, more clearly (Figure~\ref{fig:type1_power_oracle_dilute_concat_qwen_binary_multinomial}).
The eCITs with sufficient embeddings containing $Z^{Qwen}$, i.e. the oracle and the three diluted arms, stay within or close to two MC standard errors of $\alpha=0.05$ at every sample size, with rejection rates of at most $12\%$ for binary and $11\%$ for multinomial $Y$ at $n = 10{,}000$, and their power reaches at least $90\%$ at $n = 10{,}000$. 
Diluting the generating embedding with further embedding maps therefore still leads to tests controlling T1E, and only raises the sample size at which full power is reached, as under the Jina DGM. 
This also indicates that the PCM performs well on this dataset and the DGMs, also in the higher-dimensional conditioning sets, as long as the embedding of $Z$ remains sufficient. 
The T1E inflation of eCITs that omit $Z^{Qwen}$ increases with the sample size. For example, for binary $Y$ and $Z^{\omega_Z} = \text{BGE}+\text{Jina}+ \text{Llama}$ the T1E rises from $6\%$ at $n = 1{,}000$ to $48\%$ at $n = 5{,}000$ and $86\%$ at $n = 10{,}000$. 
Other DGMs show the same pattern at somewhat smaller magnitudes. Enlarging the conditioning set within this group reduces the inflation, with $Z^{\omega_Z} = \text{learned}+ \text{BGE}+\text{Jina}+\text{Llama}$ rejecting in $38\%$ of the replications at $n = 10{,}000$.
For the Qwen DGM, concatenating further embedding maps thus recovers part, but not all, of the information in $Z$ about $Y$ that $Z^{Qwen}$ carries. For the corresponding eCITs, we cannot interpret their power, because they do not control the T1E.

For continuous $Y$ the results separate the sample size behavior of the arms by the sufficiency of $Z^{\omega_Z}$ (Figures~\ref{fig:type1_power_oracle_dilute_concat_jina_binary_multinomial_continuous} and~\ref{fig:type1_power_oracle_dilute_concat_qwen_binary_multinomial}). 
The oracle eCIT stays within two MC standard errors of $\alpha = 0.05$ at every
sample size considered, as for the categorical outcomes. The diluted arms, whose $Z^{\omega_Z}$ retains the generating embedding and is therefore sufficient by construction, are inflated at $n = 1{,}000$ but calibrated at $n = 10{,}000$, so that the nuisance regressions of the PCM converge more slowly here than in the categorical case. 
The arms omitting the generating embedding remain inflated at $n = 10{,}000$. Under the Jina DGM the rejection rates of some of these arms are not monotone in $n$, consistent with the convergence of the nuisance regressions and the effect of the omitted information acting simultaneously.

Information in $Z$ about $Y$ that $Z^{\omega_Z}$ does not represent acts as an omitted conditioning variable, which condition~(ii$'$) of Theorem~\ref{thm:mean-transfer} excludes. 
It enters as conditional dependence between the embeddings, so the embedded null is false although the null on the texts holds. In addition, the nuisance regressions of the PCM, which are linear in the supplied embeddings, need no longer be correctly specified once the generating embedding is omitted. Both may contribute to the observed inflation, and the present simulation does not separate their contributions.
Enlarging $Z^{\omega_Z}$ then reduces the T1E inflation (cf.\ the T1E of the eCITs with $Z^{\omega_Z}=\text{learned}+\text{BGE}+\text{Jina}+\text{Llama}$ against $Z^{\omega_Z}=\text{learned}$ and $Z^{\omega_Z}=\text{BGE}+\text{Jina}+\text{Llama}$ under the Qwen DGM in Figure~\ref{fig:type1_power_oracle_dilute_concat_qwen_binary_multinomial}), 
but does not guarantee that the remaining inflation is negligible at a given $n$. Both DGMs omit the generating embedding in some arms, but under the Jina DGM the induced conditional dependence stays negligible at $\sqrt{n}$ scale for the categorical outcomes over the sample sizes considered. 
Whether an insufficient $Z^{\omega_Z}$ matters at a given $n$ is thus a property of the DGM and the embeddings. The simulation design assesses this for the dataset at hand before the eCITs are applied. 
In our simulation, the arms retaining the generating embedding are calibrated throughout for the categorical outcomes, and for the continuous one the oracle is calibrated throughout while the diluted arms are calibrated from $n = 10{,}000$, and the inflation of the remaining arms decreases as $Z^{\omega_Z}$ is enlarged,  motivating the full concatenation in Subsection~\ref{subsec:real_world_application}.

\subsection{Real-world case study}\label{subsec:real_world_application}

We apply the eCITs to the test sets described in Subsection~\ref{subsec:data}, with $Y$ the speaker's faction ($L = 7$) or gender ($L = 2$), $Z$ the cleaned speeches and $X$ its summary, for the Gemma 4 and Qwen 3.6 summarizers (cf.\ Appendix~\ref{app:summary_gen_details} for additional details). 
Since $Y$ is categorical, testing conditional mean independence with the PCM is equivalent to testing CI (Remark~\ref{remark:equivalence_ci_mean_cat_Y}), so nothing is lost by using a CMIT in our application. 
The specifically learned embedding maps are obtained from classifiers fitted on the training set, which is disjoint from the test set (see also Appendix~\ref{app:embedder_details}), while the pre-trained embedding maps were fitted on large corpora that may overlap with the German Parliament speeches. 
Treating all embedding maps as independent of the test set and the observations as i.i.d., we have $\hat\theta = (\hat\beta,\hat\gamma) \indep S$, so the conditions on $\hat\beta$ (Assumption~\ref{ass:beta-mean}) and $\hat\gamma$ (Assumption~\ref{ass:mean-gamma}) in $\mathcal{P}_0^{\mathrm{suff, mean}}$ hold and Corollary~\ref{cor:mean-iid} applies with the i.i.d.\ null class, over which the PCM is asymptotically valid. 
Neither assumption is guaranteed here, and we return to both below, together with the sufficiency of $Z^{\omega_Z}$, which cannot be verified and is instead examined through the simulation study in Subsection~\ref{subsec:simulation} and by varying $Z^{\omega_Z}$ across the settings reported here.

\begin{table}[t!]
\centering
\caption{Test statistics $T$ and $p$-values of the eCITs on the test sets of Subsection~\ref{subsec:data}, for the speaker's gender and faction and the summaries of both summarizers. Gemma 4: gender $n=50{,}839$, faction $n=53{,}978$. Qwen 3.6: gender $n=50{,}796$, faction $n=53{,}855$; the sample sizes differ slightly because a few summary generations were empty. $^{*}$ indicates $p<0.05$.
$Z^{\omega_Z}$ and $X^{\omega_X}$ give the embeddings of the speech and of the summary; pred-hidden denotes the masked mean of the last hidden layer of the classifier predicting $Y$, and Qwen the Qwen3 embedder, as distinct from the Qwen 3.6 summarizer. 
For the PCM, the nuisance regressions on the training dataset are $L_2$-penalized and those on the evaluation dataset unpenalized, the regressions of $Y$ being multinomial logistic and those of the fitted probabilities and of $\hat f$ least squares, with the ridge constant $\hat c$ in the projection direction tuned on $D_1$ (Appendix~\ref{app:catpcm:test}).}
\label{tab:cit_tuned_primary_xy_hh}
\resizebox{\textwidth}{!}{%
\begin{tabular}{llllrrrr}
\toprule
 Outcome & $Z^{\omega_Z}$ & $X^{\omega_X}$ & Gemma 4 $T$ & Gemma 4 $p$ & Qwen 3.6 $T$ & Qwen 3.6 $p$ \\
\midrule
 gender & BGE+Jina+Qwen+Llama & BGE+Jina+Qwen+Llama & $24.687$ & $<10^{-15}$${}^{*}$ & $34.146$ & $<10^{-15}$${}^{*}$ \\
 faction & BGE+Jina+Qwen+Llama & BGE+Jina+Qwen+Llama & $34.705$ & $<10^{-15}$${}^{*}$ & $51.317$ & $<10^{-15}$${}^{*}$ \\
 gender & pred-hidden+BGE+Jina+Qwen+Llama & pred-hidden+BGE+Jina+Qwen+Llama & $4.194$ & $1.37\times 10^{-5}$${}^{*}$ & $5.505$ & $1.85\times 10^{-8}$${}^{*}$ \\
 faction & pred-hidden+BGE+Jina+Qwen+Llama & pred-hidden+BGE+Jina+Qwen+Llama & $11.859$ & $<10^{-15}$${}^{*}$ & $18.334$ & $<10^{-15}$${}^{*}$ \\
 gender & pred-hidden+BGE+Jina+Qwen+Llama & pred-hidden & $3.513$ & $2.21\times 10^{-4}$${}^{*}$ & $5.384$ & $3.64\times 10^{-8}$${}^{*}$ \\
 faction & pred-hidden+BGE+Jina+Qwen+Llama & pred-hidden & $10.664$ & $<10^{-15}$${}^{*}$ & $17.495$ & $<10^{-15}$${}^{*}$ \\
 gender & pred-hidden & pred-hidden & $4.638$ & $1.76\times 10^{-6}$${}^{*}$ & $5.980$ & $1.12\times 10^{-9}$${}^{*}$ \\
 faction & pred-hidden & pred-hidden & $12.683$ & $<10^{-15}$${}^{*}$ & $20.010$ & $<10^{-15}$${}^{*}$ \\
\bottomrule
\end{tabular}%
}
\end{table}

We report four settings (Table~\ref{tab:cit_tuned_primary_xy_hh}), which differ in how much of the speech and the summary the embeddings are able to represent.
First, we concatenate the four pre-trained embedders (BGE, Jina, Llama, Qwen) for both $Z$ and $X$, giving $q=p = 1120$, leading to eCITs with diverse pre-trained embeddings but without specifically trained ones.
Second, we additionally include the  learned embeddings for $X$ and $Z$, 
giving $q=p = 1888$. This setting includes all considered embeddings, and aims to minimize residual insufficiency of the embedding of $Z$. Specifically, the learned representation of the speech is trained to extract the information in $Z$ about $Y$, and the pre-trained embeddings are general representations of $Z$, which together aim to satisfy condition~(ii$'$), i.e. retaining the conditional mean information of $Z$ about $Y$.
Third, we omit the pre-trained embeddings of the second setting from $X^{\omega_X}$, giving $p=1888$ and $ q = 768$, evaluating a potential power decrease under a less representative embedding of $X$, and addressing potential violations of Assumption~\ref{ass:beta-mean}.
Fourth, we use only the learned embeddings for $X$ and $Z$, giving $q = p = 768$. This setting drops the embedders whose training corpora may overlap with the German Parliament speeches, addressing potential violations of Assumption~\ref{ass:beta-mean} and Assumption~\ref{ass:mean-gamma}.

All sixteen tests reject at $\alpha = 0.05$.
In the first setting the test statistics are large, at $T = 24.7$ (gender) and $T = 34.7$ (faction) for the Gemma summaries, and even larger for Qwen (Table~\ref{tab:cit_tuned_primary_xy_hh}).
Adding the learned representations to $Z^{\omega_Z}$ and $X^{\omega_X}$ in the second setting reduces the test statistics strongly. The two settings differ only by the learned representations, so the reduction shows that the learned representation of the speech carries information about $Y$ that the pre-trained embeddings do not, and hence that the first conditioning set is insufficient and at least part of the conditional association detected in this setting is attributable to that insufficiency.
Reducing $X^{\omega_X}$ to the learned representation in the third setting lowers the statistics slightly further, in line with the smaller signal retained by a less representative embedding of the summary.
Dropping the pre-trained embedders from $Z^{\omega_Z}$ in the fourth setting increases the statistics only slightly, which suggests that these embedders carry little about $Y$ beyond what the learned representations already supply.
The test decision does not change across the settings: the smallest statistic across all settings, outcomes and summarizers is $T = 3.5$, corresponding to $p = 2.2 \times 10^{-4}$, and the faction statistics remain above $10$ throughout.
We therefore find that the summaries of both LLMs contain information about the speaker's faction and gender beyond the speech they were generated from.
The statistics are larger for Qwen than for Gemma in all eight comparisons, but $T$ depends on the sufficiency of the embedding maps and the nuisance fits as well as on the strength of the conditional dependence, so we do not read this as a comparison of the two summarizers.

We discuss four explanations for the test decision other than conditional dependence.
First,  the embedding of $Z$ could be insufficient: exact sufficiency is not plausible for embedding maps learned from data, and by Theorem~\ref{thm:mean-transfer} any information in $Z$ about the conditional mean of $Y$ that $Z^{\omega_Z}$ drops can act as an omitted conditioning variable (cf.\ also the Qwen DGM results in Subsection~\ref{subsec:simulation}).
However, the eCITs with a large $Z^{\omega_Z}$ consisting of concatenated pre-trained and learned embeddings, and thus covering both general and $Y$-specific information in $Z$, still reject (second setting).
Furthermore, adding the four pre-trained embedders to a $Z^{\omega_Z}$ containing only learned embeddings, while holding $X^{\omega_X} = \text{pred-hidden}$ fixed, changes the test statistics only slightly and leaves every statistic well above the critical value.
The pre-trained embedders thus carry little information about $Y$ beyond what the learned representation of the speech already supplies, which suggests that the bias from residual insufficiency is not large here.
The corresponding simulation arms are calibrated for categorical $Y$ up to $n = 10{,}000$, but this is limited evidence at the sample sizes used here, since the inflation caused by an insufficient $Z^{\omega_Z}$ grows with $n$.
We can therefore not exclude residual insufficiency, although it would have to concern information that neither the pre-trained embedders nor a classifier trained for the task recovers from the speech.

Second, the nuisance regressions on which the validity of the PCM rests could be estimated inaccurately.
The dimension of the conditioning set, between $p = 768$ and $p = 1888$, is large relative to the $n_2 \approx 26{,}000$ observations in the evaluation half.
However, this is probably not the main driver of the observed statistics, since they do not increase with $p$, whereas the corresponding bias term does.
In particular, the largest statistics are obtained at $p = 1120$, while the settings with the highest and the lowest dimension, $p = 1888$ and $p = 768$, give smaller values of comparable size.
Furthermore, the simulation study uses conditioning sets of comparable dimensions at the smaller sample size $n = 10{,}000$, hence at a less favorable ratio of $p$ to $n_2$ than here, and the arms with sufficient $Z^{\omega_Z}$ remain calibrated there.

Third, the detected conditional dependence could run through the embedding maps rather than the summarizers.
The pre-trained embedders are fitted on external corpora, which may themselves contain German Parliament speeches, and the classifier supplying the learned representations is fine-tuned from a pre-trained encoder. 
The split is drawn at the level of speeches, so the training and test samples share speakers,  topics and electoral terms, which can induce dependence between train and test sample.
Thus, $\hat\beta$ and $\hat\gamma$ might not be independent of the test sample, and both Assumption~\ref{ass:beta-mean} and Assumption~\ref{ass:mean-gamma} in the definition of $\mathcal{P}_0^{\mathrm{suff,mean}}$ can fail.
The embeddings could then contain information about $Y$ from the embedder rather than the summarizer, so that a rejection need not reflect conditional dependence at the level of the texts.
For this to produce a rejection, the effect has to differ between the speech and its summary, since information that the embedder attaches to both is conditioned on through $Z^{\omega_Z}$.
Fine-tuning the learned maps on a holdout split removes the dependence on the test speeches themselves, but neither this nor the choice of split removes the dependence induced by the pre-training corpus of the encoder, so the fourth setting reduces but does not remove this dependence.
The agreement across the two summarizers does not speak against this explanation, as the same embedders are used for both.

Fourth, the observations could be dependent rather than i.i.d.
Speakers recur across speeches and speeches are dependent through shared topics and across time, so the rows of $S$ are not i.i.d.\ and the attributes $Y$ are constant within speaker.
Lemmas~\ref{lem:lift} and \ref{lem:lifting-mean} and the asymptotic validity of the PCM all rely on the i.i.d.\ assumption, and positive dependence between observations can make the PCM anti-conservative.
We can therefore not exclude that this dependence contributes to the observed magnitude of the statistics; extending the eCITs to dependent samples is left for future work.

The result is consistent with the mechanism of Example~\ref{ex:knn}.
Both summarizers were pre-trained on web-scale corpora that plausibly include German Parliament plenary protocols, so speeches in the test set, speeches by the same speakers, or other texts associated with them may have been seen during training.
Memorization of these texts then allows the summary to contain information about speaker attributes which is not in the speech.
Under all four combinations of $Z^{\omega_Z}$ and $X^{\omega_X}$ the eCITs reject for both outcomes and both summarizers, and the changes in the test statistics across the settings follow the expected direction, which weakens some of the alternative explanations for the test decisions.
The test thus certifies that the summaries carry additional information about the protected attribute on this task and dataset, with a T1E guarantee under the stated conditions.
It does not guarantee sufficiency of $Z^{\omega_Z}$, localize the information, quantify its magnitude on an interpretable scale, or attribute it to a specific mechanism, including the memorization discussed here.

\section{Conclusion}

We introduced embedded conditional independence tests (eCITs) to test for conditional dependence between a text and a scalar given a text, a setting in which existing CITs cannot be applied with T1E guarantees.
We first justified CI testing as a bias criterion for LLM outputs and characterized when the null and alternative can hold for a fixed LLM, illustrating that memorization of individual observations can induce conditional dependence. Existing bias metrics do not provide such guarantees: intrinsic metrics need not correlate with downstream behavior, extrinsic metrics are bound to benchmark datasets with specific characteristics, and LLM-based evaluators can themselves be biased. In contrast, our criterion requires neither a gold-standard output nor an annotation of what an unbiased output would be: the source text supplies the baseline against which additional information is measured.
We then characterized sufficient embedding maps for the conditioned text, under which a valid CIT for the embedded hypothesis is a valid CIT for the textual hypothesis, and gave conditions under which the two hypotheses are equivalent. We further showed that the requirement on the embedding maps can be weakened to mean sufficiency when the embedded test targets conditional mean independence, and that for categorical $Y$, as in our application, the two nulls coincide for the i.i.d.\ laws considered here.
We further proposed a simulation design to assess T1E control and power of the eCITs on a given dataset before applying them. On German Parliament speeches, the design exhibits both regimes it is meant to detect: the eCITs are calibrated for some combinations of embedding maps, while others show T1E inflation that need not vanish, and can grow, with the sample size. Since validity depends on the sufficiency of the embedding map under a given data generating process, the design serves to assess sufficiency for the dataset and embedding maps at hand.
Applying the eCITs to the speeches and the summaries of two LLMs, we reject the null for both the speaker's faction and gender. Four alternative explanations cannot be excluded, 
although 
the calibration of comparable arms in the simulation and the changes in test statistics with consistent test decisions both speak against any of them accounting for the rejections on its own.

The main open challenge is exact (mean) sufficiency, which is unlikely to hold for pre-trained or learned embedding maps, and even small amounts of residual information can affect the T1E at larger sample sizes, as the simulation suggests. A diagnostic quantifying how much relevant information about $Y$ or $X^{\omega_X}$  is left in the source text, rather than certifying that none is, would be of considerable interest.
Two further directions remain. We studied only embedding maps learned on a holdout set; (unsupervised) in-sample learning and the dependence it introduces in the test sample deserve theoretical and empirical study. 
And while we used the PCM, 
other CITs could be combined with sufficient embedding maps in the same way.

eCITs make conditional independence testing available for multimodal data and thereby give statistical guarantees for bias claims about LLM outputs on a specified task and dataset. While our focus has been on summarization and text, our eCITs apply to any setting in which the conditioning variable and the tested variable admit representations satisfying the (mean) sufficiency conditions.

\acks{Funded by the Deutsche Forschungsgemeinschaft (DFG, German Research Foundation) - project number 459422098.
}

\section*{Use of AI Tools}

The LLMs analysed in this work are described in Appendix~\ref{app:app_details}; their outputs are the object of study rather than part of the manuscript. 
In preparing the manuscript, the authors additionally used LLMs to assist with drafting and language editing. The authors take full responsibility for the contents of the paper, including all derivations, results and references.

\clearpage
\newpage

\appendix

\section{Proofs}\label{app:proofs}

The proofs are based on applying the lemmata from \citet{dawid1979conditional}. For readability, we refer to them as specified in Table \ref{tab:dawid_ci_statements}. 
\begin{table}[t!]
    \centering
    \begin{tabular}{c|c|c}
    \citet{dawid1979conditional} & Statement & Name \\
    \hline
       Lemma 4.1 & $A\indep B\mid C\iff A\indep (B,C)\mid C$ & redundancy \\
       Lemma 4.2 (i) & $A\indep B\mid C\implies A\indep f(B)\mid C$ & decomposition\\
       Lemma 4.2 (ii) & $A\indep B\mid C \text{ and } D=f(B)\implies A\indep B\mid (C,D)$ & weak union \\
       Lemma 4.3 & $A\indep B\mid C \text{ and } D\indep B\mid (A,C)\implies (A,D)\indep B\mid C$ & contraction
    \end{tabular}
    \caption{ $f$ denotes any measurable function of $B$.}
    \label{tab:dawid_ci_statements}
\end{table}

The general form of Lemma \ref{lem:disintegration} is given in the following:
\begin{lemma}[Disintegration of conditional independence]
\label{lem:disintegration-general}
Let $(E_A,\mathcal{E}_A)$, $(E_B,\mathcal{E}_B)$, $(E_C,\mathcal{E}_C)$ be
standard Borel spaces and write
$(E,\mathcal{E}) = (E_A\times E_B\times E_C,\,
\mathcal{E}_A\otimes\mathcal{E}_B\otimes\mathcal{E}_C)$, with coordinate
projections $A,B,C$. Let
\[
  \mathcal{P}_0 := \bigl\{\,\mu \text{ probability measure on }(E,\mathcal{E})
  \;:\; A \indep B \mid C \text{ under } \mu \,\bigr\}.
\]
Let $V$ be an $E$-valued and $W$ a $(\mathcal{W},\mathcal{G})$-valued random
variable on $(\Omega,\mathcal{F},\mathbb{P})$, where $(\mathcal{W},\mathcal{G})$
is an arbitrary measurable space, and let $\mathbb{P}^{V\mid W}$ be a regular
conditional distribution of $V$ given $W$. Then
$\{w\in\mathcal{W}:\mathbb{P}^{V\mid W=w}\in\mathcal{P}_0\}\in\mathcal{G}$ and
\[
  A \indep B \mid (C,W)
  \quad\Longleftrightarrow\quad
  \mathbb{P}^{V\mid W}\in\mathcal{P}_0 \quad \mathbb{P}^{W}\text{-a.s.}
\]
\end{lemma}

\begin{proof}
For a bounded measurable $h:E\to\mathbb{R}$ choose, by the Doob--Dynkin lemma, a
jointly measurable $\Phi_h:\mathcal{W}\times E_C\to\mathbb{R}$ with
$\mathbb{E}[h(V)\mid C,W]=\Phi_h(W,C)$ $\mathbb{P}$-a.s.; this requires no
assumption on $(\mathcal{W},\mathcal{G})$.

\emph{Step 1 (two-stage conditioning).} We show that for
$\mathbb{P}^W$-a.e.\ $w$, the map $\Phi_h(w,\cdot)$ is a version of
$\mathbb{E}_{\mathbb{P}^{V\mid W=w}}[h\mid C]$. Fix $D\in\mathcal{E}_C$ and set
\[
  \psi_D(w) := \int \bigl(\Phi_h(w,c)-h(v)\bigr)\mathbf{1}_D(c)\,
  \mathbb{P}^{V\mid W=w}(dv), \qquad v=(a,b,c),
\]
which is measurable in $w$ by joint measurability of $\Phi_h$ and of
$w\mapsto\mathbb{P}^{V\mid W=w}(F)$, $F\in\mathcal{E}$. By the defining property
of the regular conditional distribution,
$\psi_D(W)=\mathbb{E}\bigl[(\Phi_h(W,C)-h(V))\mathbf{1}_D(C)\mid W\bigr]$ a.s.
For every $G\in\mathcal{G}$ the variable $\mathbf{1}_D(C)\mathbf{1}_G(W)$ is
$\sigma(C,W)$-measurable, whence
$\mathbb{E}[(\Phi_h(W,C)-h(V))\mathbf{1}_D(C)\mathbf{1}_G(W)]=0$ and therefore
$\mathbb{E}[\psi_D(W)\mathbf{1}_G(W)]=0$. As $\psi_D(W)$ is
$\sigma(W)$-measurable, $\psi_D(W)=0$ a.s. Since $E_C$ is standard Borel,
$\mathcal{E}_C$ is countably generated; applying the above to a countable
generating algebra, discarding the union of the corresponding null sets and
extending by a monotone class argument for each fixed $w$ yields the claim.

\emph{Step 2 (reduction to a countable family).} Let $\{A_k\}_{k\in\mathbb{N}}$
and $\{B_l\}_{l\in\mathbb{N}}$ be countable generating algebras of
$\mathcal{E}_A$ and $\mathcal{E}_B$, and abbreviate
$f_k := \mathbf{1}_{A_k}\circ A$, $g_l := \mathbf{1}_{B_l}\circ B$. By a monotone
class argument, $A\indep B\mid(C,W)$ holds if and only if
\begin{equation}\label{eq:factorisation}
  \Phi_{f_kg_l} = \Phi_{f_k}\,\Phi_{g_l}
\end{equation}
holds at $(W,C)$ $\mathbb{P}$-a.s.\ for all $k,l$; and, by Step 1, for
$\mathbb{P}^W$-a.e.\ $w$ the measure $\mathbb{P}^{V\mid W=w}$ lies in
$\mathcal{P}_0$ if and only if \eqref{eq:factorisation} holds at $(w,C)$
$\mathbb{P}^{V\mid W=w}$-a.s.\ for all $k,l$.

\emph{Step 3 (transfer).} Let
\[
  N_{k,l} := \bigl\{(w,v)\in\mathcal{W}\times E \;:\;
  \Phi_{f_kg_l}(w,c) \neq \Phi_{f_k}(w,c)\Phi_{g_l}(w,c) \bigr\}
  \in \mathcal{G}\otimes\mathcal{E},
\]
with sections $(N_{k,l})_w$. The regular conditional distribution gives
\[
  \mathbb{P}\bigl((W,V)\in N_{k,l}\bigr)
  = \int \mathbb{P}^{V\mid W=w}\bigl((N_{k,l})_w\bigr)\,\mathbb{P}^W(dw),
\]
so $\mathbb{P}((W,V)\in N_{k,l})=0$ if and only if
$\mathbb{P}^{V\mid W=w}((N_{k,l})_w)=0$ for $\mathbb{P}^W$-a.e.\ $w$. Taking the
union over the countably many pairs $(k,l)$ and combining with Step 2 proves the
equivalence. Finally,
\[
  \bigl\{w:\mathbb{P}^{V\mid W=w}\in\mathcal{P}_0\bigr\}
  = \bigcap_{k,l}\bigl\{w: \mathbb{P}^{V\mid W=w}\bigl((N_{k,l})_w\bigr)=0\bigr\}
  \in\mathcal{G},
\]
since $w\mapsto\mathbb{P}^{V\mid W=w}((N_{k,l})_w)$ is measurable for each
$(k,l)$ by Fubini's theorem for kernels.
\end{proof}

\begin{lemma}[Disintegration of conditional mean independence]\label{lem:disint-mean}
Let $(E_A,\mathcal E_A),(E_B,\mathcal E_B),(E_C,\mathcal E_C)$ be standard Borel with
$E_B\subseteq\mathbb R^{d}$ Borel, and write
$(E,\mathcal E)=(E_A\times E_B\times E_C,\mathcal E_A\otimes\mathcal E_B\otimes\mathcal E_C)$
with coordinate projections $A,B,C$. Let
\[
  \mathcal P_0^{\mathrm{mean}}:=\Big\{\mu \text{ probability measure on } (E,\mathcal E):
  \int\|b\|\,\mu(dv)<\infty,\
  \mathbb E_\mu[B\mid A,C]=\mathbb E_\mu[B\mid C]\ \mu\text{-a.s.}\Big\}.
\]
Let $V$ be an $E$-valued and $W$ a $(\mathcal W,\mathcal G)$-valued random variable on
$(\Omega,\mathcal F,\mathbb P)$ with $\mathbb E\|B(V)\|<\infty$, where $(\mathcal W,\mathcal G)$
is an arbitrary measurable space, and let $\mathbb P^{V\mid W}$ be a regular conditional
distribution of $V$ given $W$. Then
$\{w\in\mathcal W:\mathbb P^{V\mid W=w}\in\mathcal P_0^{\mathrm{mean}}\}\in\mathcal G$ and
\[
  \mathbb E[B\mid A,C,W]=\mathbb E[B\mid C,W]\ \text{a.s.}
  \iff
  \mathbb P^{V\mid W}\in\mathcal P_0^{\mathrm{mean}}\ \ \mathbb P^{W}\text{-a.s.}
\]
\end{lemma}
\begin{proof}
It suffices to treat $d=1$, applying the result to each coordinate of $B$ and
intersecting the countably many resulting null sets.

By the Doob--Dynkin lemma choose jointly measurable
$\Phi_1:\mathcal W\times E_A\times E_C\to\mathbb R$ and
$\Phi_2:\mathcal W\times E_C\to\mathbb R$ with
$\mathbb E[B\mid A,C,W]=\Phi_1(W,A,C)$ and $\mathbb E[B\mid C,W]=\Phi_2(W,C)$
$\mathbb P$-a.s.; this requires no assumption on $(\mathcal W,\mathcal G)$.

\emph{Step 1.} Step~1 of Lemma~\ref{lem:disintegration-general} applies verbatim to integrable
rather than bounded $h$: by Fubini's theorem for kernels,
$\mathbb E\|B(V)\|<\infty$ gives $\int\|b\|\,\mathbb P^{V\mid W=w}(dv)<\infty$ for
$\mathbb P^W$-a.e.\ $w$, and dominated convergence replaces boundedness in the monotone
class step. Applying it once with conditioning field $\sigma(A,C)$ and once with
$\sigma(C)$ yields that, for $\mathbb P^W$-a.e.\ $w$, the maps $\Phi_1(w,\cdot,\cdot)$ and
$\Phi_2(w,\cdot)$ are versions of $\mathbb E_{\mathbb P^{V\mid W=w}}[B\mid A,C]$ and
$\mathbb E_{\mathbb P^{V\mid W=w}}[B\mid C]$, respectively. In particular
$\mathbb P^{V\mid W=w}\in\mathcal P_0^{\mathrm{mean}}$ if and only if
$\Phi_1(w,A,C)=\Phi_2(w,C)$ holds $\mathbb P^{V\mid W=w}$-a.s.

\emph{Step 2.} Let
$N:=\{(w,v)\in\mathcal W\times E:\Phi_1(w,a,c)\neq\Phi_2(w,c)\}\in\mathcal G\otimes\mathcal E$,
with sections $N_w$. The regular conditional distribution gives
\[
  \mathbb P\big((W,V)\in N\big)=\int \mathbb P^{V\mid W=w}(N_w)\,\mathbb P^W(dw),
\]
so $\mathbb P((W,V)\in N)=0$ if and only if $\mathbb P^{V\mid W=w}(N_w)=0$ for
$\mathbb P^W$-a.e.\ $w$. Combined with Step~1 this is the asserted equivalence. Finally
$\{w:\mathbb P^{V\mid W=w}\in\mathcal P_0^{\mathrm{mean}}\}
=\{w:\mathbb P^{V\mid W=w}(N_w)=0\}\cap\{w:\int\|b\|\,\mathbb P^{V\mid W=w}(dv)<\infty\}
\in\mathcal G$,
since $w\mapsto \mathbb P^{V\mid W=w}(N_w)$ and
$w\mapsto\int\|b\|\,\mathbb P^{V\mid W=w}(dv)$ are measurable by Fubini's theorem for
kernels.
\end{proof}

\begin{lemma}[Lifting to the sample level]\label{lem:lift}
Let $(X_i,Y_i,Z_i)_{i=1}^n$ be i.i.d.\ copies of $(X,Y,Z)$ with values in the
standard Borel spaces $(\mathcal X,\mathcal F_{\mathcal X})$,
$(\mathcal Y,\mathcal F_{\mathcal Y})$, $(\mathcal Z,\mathcal F_{\mathcal Z})$.
Then
\[
  X\indep Y\mid Z
  \qquad\Longleftrightarrow\qquad
  X^n\indep Y^n\mid Z^n .
\]
\end{lemma}
\begin{proof}[Proof of Lemma~\ref{lem:lift}]
``$\Rightarrow$''. Since the rows $(X_i,Y_i,Z_i)$ are independent, for bounded
measurable $h_i:\mathcal X\times\mathcal Y\to\mathbb R$ we have
\begin{equation}\label{eq:row-factorisation}
  \mathbb E\Big[\textstyle\prod_{i=1}^n h_i(X_i,Y_i)\,\Big|\,Z^n\Big]
  = \prod_{i=1}^n \mathbb E\big[h_i(X_i,Y_i)\mid Z_i\big]
  \qquad \text{a.s.}
\end{equation}
Indeed, the right-hand side is $\sigma(Z^n)$-measurable, and for
$\phi(Z^n)=\prod_i\phi_i(Z_i)$ with bounded measurable $\phi_i$ both sides have
the same expectation against $\phi(Z^n)$ by independence across $i$; such
$\phi$ generate a $\pi$-system generating $\sigma(Z^n)$, so
\eqref{eq:row-factorisation} follows by a functional monotone class argument.

Now let $f_i,g_i$ be bounded and measurable and apply
\eqref{eq:row-factorisation} with $h_i(x,y)=f_i(x)g_i(y)$. By $X\indep
Y\mid Z$ applied row-wise,
$\mathbb E[f_i(X_i)g_i(Y_i)\mid Z_i]=\mathbb E[f_i(X_i)\mid Z_i]\,
\mathbb E[g_i(Y_i)\mid Z_i]$, so regrouping the product and applying
\eqref{eq:row-factorisation} twice more, once with $g_i\equiv 1$ and once with
$f_i\equiv 1$, gives
\[
  \mathbb E\Big[\textstyle\prod_i f_i(X_i)\prod_i g_i(Y_i)\,\Big|\,Z^n\Big]
  = \mathbb E\Big[\textstyle\prod_i f_i(X_i)\,\Big|\,Z^n\Big]\,
    \mathbb E\Big[\textstyle\prod_i g_i(Y_i)\,\Big|\,Z^n\Big].
\]
Functions of the form $\prod_i f_i$ with $f_i=\mathbf 1_{A_i}$, $A_i\in\mathcal
F_{\mathcal X}$, form a $\pi$-system generating $\mathcal F_{\mathcal
X}^{\otimes n}$, and analogously for $Y^n$; a functional monotone class argument
in each argument separately extends the display to all bounded
$\mathcal F_{\mathcal X}^{\otimes n}$- and $\mathcal F_{\mathcal
Y}^{\otimes n}$-measurable functions, which is $X^n\indep Y^n\mid Z^n$.

``$\Leftarrow$''. Decomposition applied twice gives $X_1\indep Y_1\mid Z^n$.
Since $(X_1,Y_1,Z_1)\indep Z_{2:n}$, redundancy and contraction give
$X_1\indep Y_1\mid Z_1$, and $(X_1,Y_1,Z_1)\overset{d}{=}(X,Y,Z)$.
\end{proof}

\begin{lemma}[Lifting to the sample level, mean version]\label{lem:lifting-mean}
Let $(X_i,Y_i,Z_i)_{i=1}^n$ be i.i.d.\ copies of $(X,Y,Z)$ with values in the standard
Borel spaces $(\mathcal X,\mathcal F_{\mathcal X})$, $(\mathcal Y,\mathcal F_{\mathcal Y})$,
$(\mathcal Z,\mathcal F_{\mathcal Z})$, with $\mathcal Y\subseteq\mathbb R^K$ Borel and
$\mathbb E\|Y\|<\infty$. Then
\[
  \mathbb E[Y\mid X,Z]=\mathbb E[Y\mid Z]\ \text{a.s.}
  \iff
  \mathbb E[Y_i\mid X^n,Z^n]=\mathbb E[Y_i\mid Z^n]\ \text{a.s. for all } i\le n .
\]
\end{lemma}
\begin{proof}
Fix $i\le n$. Since $(X_i,Y_i,Z_i)$ is independent of $(X,Z)^n_{-i}$, weak union applied
with the measurable function $(X_i,Z_i)$ of $(X_i,Y_i,Z_i)$, followed by decomposition,
gives $Y_i\indep (X,Z)^n_{-i}\mid (X_i,Z_i)$ and $Y_i\indep Z^n_{-i}\mid Z_i$. Hence
\begin{equation}\label{eq:lifting-mean-collapse}
  \mathbb E[Y_i\mid X^n,Z^n]=\mathbb E[Y_i\mid X_i,Z_i],
  \qquad
  \mathbb E[Y_i\mid Z^n]=\mathbb E[Y_i\mid Z_i]
  \quad\text{a.s.}
\end{equation}
Both conditional expectations are well defined since $\mathbb E\|Y\|<\infty$. The two
right-hand sides agree a.s.\ for every $i$ if and only if
$\mathbb E[Y\mid X,Z]=\mathbb E[Y\mid Z]$ a.s., because
$(X_i,Y_i,Z_i)\stackrel{d}{=}(X,Y,Z)$; by \eqref{eq:lifting-mean-collapse} this is the
asserted equivalence.
\end{proof}

\printProofs

\section{Application}\label{app:app_details}

\subsection{Source Texts, LLMs, and Summary Generation}\label{app:summary_gen_details}

We clean the German Parliament speeches from \cite{richter2020opendiscourse} by removing parenthetical text such as applause or laughter, deleting common procedural phrases (e.g., “Das Wort hat”), stripping out text between dashes and extraneous whitespace, and normalizing quotation marks. 
Furthermore, we consider speeches as valid after filtering by minimum word count (50) and minimum sentences (2), and excluding speeches containing procedural indicators. 
For each speaker, we merge consecutive speeches if no valid speech from another speaker intervenes. We treat invalid interruptions from other speakers as [UNTERBRECHUNG/REDE:...] appended to the main speech. We randomly sample cleaned speeches and manually inspect for remaining procedural or invalid content.

Next, we summarize the cleaned speeches $Z$ with the LLM Gemma 4 12B \citep{gemmateam2026gemma4}, used throughout the simulation study (Subsection~\ref{subsec:simulation}) and the real-world application (Subsection~\ref{subsec:real_world_application}). We allow for a maximum of 256 generated tokens per summary (3–4 sentences) and use greedy decoding so that the summaries are reproducible.
For the real-world application (Subsection~\ref{subsec:real_world_application}), we also summarize the source texts $Z$ with the LLM Qwen 3.6 27B \citep{qwen3.6-27b}. We again allow for a maximum of 256 generated tokens per summary (3–4 sentences) and sample with temperature $0.7$, top-$p$ $0.8$, and top-$k$ $20$, the recommended decoding parameters for the model’s non-thinking (Instruct) mode.

The summaries $X$ are generated by prompting the LLM with \texttt{Du bist ein erfahrener politischer Analyst. Deine Aufgabe ist es, Reden aus dem deutschen Bundestag prägnant und objektiv in einem Fließtext zusammenzufassen. Fokussiere auf die Hauptargumente und politischen Positionen.
Die mit '$(\{\backslash \backslash d+\})$' gekennzeichneten Bemerkungen im Text beziehen sich auf mögliche Zwischenrufe oder kleinere Beiträge von anderen Abgeordneten.
Die Rede kann außerdem [REDE] und [UNTERBRECHUNG] Marker enthalten. [REDE] markiert Teile der Hauptrede. [UNTERBRECHUNG] markiert Zwischenkommentare oder Fragen. WICHTIG: Gib NUR die Zusammenfassung zurück, ohne einleitende Phrasen wie 'Hier ist eine Zusammenfassung' oder Ähnliches.}

\subsection{Embedding Maps}\label{app:embedder_details}

We embed speeches $Z$ and summaries $X$ with four multilingual models, using the same encoders in the simulation study (Subsection~\ref{subsec:simulation}) and the real-world application (Subsection~\ref{subsec:real_world_application}): BGE-M3 \citep{chen2024m3} at its native $1024$ dimensions; jina-embeddings-v5-text-small \citep{akram2026jinaembeddingsv5texttasktargetedembeddingdistillation} with the classification adapter, truncated to $32$ dimensions; Llama Nemotron Embed 1B v2 \citep{nvidia_llama_nemotron_embed_1b_v2}, truncated to $32$ dimensions; and Qwen3-Embedding-8B \citep{qwen3embedding}, likewise truncated to $32$ dimensions.
The three $32$-dimensional representations use Matryoshka truncation of the native embedding \citep{kusupati2022matryoshka}. Each embedding is $\ell_2$-normalized. We then column-standardize the coordinates (train-fitted mean and standard deviation) so that features enter the simulation DGP for $Y$ on a comparable scale, and because $\ell_2$-penalized PCM nuisances are sensitive to the scale of $X$ and $Z$.

As additional representations we use learned embedding maps: we fit the ModernGBERT 134M \citep{wunderle2025new}, a German encoder with $768$-dimensional hidden states, to predict $Y$ from the raw speech or summary, and take the masked mean of the last hidden layer. In the real-world application these models predict faction or gender on the training split, and the $768$-dimensional hidden states on the test split are supplied to the CIT, either alone or concatenated with the pretrained embeddings. In the simulation study the same architecture is trained to predict the simulated outcome, and the resulting hidden states are used in the simulation.

In the simulation study we generate $Y$ from a single $32$-dimensional embedding of the Gemma summaries (Jina or Qwen3) and supply the CIT with the following concatenations of $Z^{\omega_Z}$ and $X^{\omega_X}$. Oracle uses only that DGP pair: $Z^{\widetilde{\omega}_Z},X^{\widetilde{\omega}_X}\in\mathbb{R}^{32}$. 
The diluted eCITs prepend the generating embedder to three further concatenations: BGE with the two remaining $32$-dimensional models ($Z^{\omega_Z},X^{\omega_X}\in\mathbb{R}^{1120}$); the $768$-dimensional learned speech hidden state, BGE, and those two models as $Z^{\omega_Z}\in\mathbb{R}^{1888}$, and the learned summary hidden state with those two models as $X^{\omega_X}\in\mathbb{R}^{864}$; and the learned hidden states alone ($Z^{\omega_Z},X^{\omega_X}\in\mathbb{R}^{800}$).
The eCITs omitting the generating embedder use the same three stacks without the DGP embedder: BGE with the two remaining $32$-dimensional models  ($Z^{\omega_Z},X^{\omega_X}\in\mathbb{R}^{1088}$); the learned speech hidden state, BGE, and those two models as $Z^{\omega_Z}\in\mathbb{R}^{1856}$, and the learned summary hidden state with those two models as $X^{\omega_X}\in\mathbb{R}^{832}$; and the learned hidden states alone ($Z^{\omega_Z},X^{\omega_X}\in\mathbb{R}^{768}$).

\subsection{Outcome Generation in the Simulation Study}
\label{app:dgm_details}

We set $\alpha_Z = \alpha_X = 1$ throughout. For the continuous outcome, the noise variance is $\sigma^2 = (1-r^2)/r^2$ with $r^2 = 0.4$, so that $s_Z$ explains 40\% of the variance of $Y$ under the null; for the binary and multinomial outcomes the noise is determined by the link, so that the three outcome types are not on a common signal scale.

Write $\bar s_Z$ for the training-split mean of the speech score. The binary intercept is $b =-\alpha_Z \bar s_Z$. For the multinomial outcome, the reference class has logit $\eta_{i,0} = 0$ and the remaining logits are
\[
  \eta_{i,k} = \cos\!\Bigl(\tfrac{2\pi k}{L}\Bigr) \alpha_Z (s_{Z,i} - \bar s_Z)
  + c\,\sin\!\Bigl(\tfrac{2\pi k}{L}\Bigr) \alpha_X s_{X,i},
  \qquad k = 1, \dots, L-1,
\]
so that the speech and the summary score load on the classes in orthogonal directions. The logits are clipped to $[-10, 10]$ before the softmax, and the resulting probabilities are floored at $10^{-6}$ and renormalized.

\section{Detailed Simulation Results}\label{app:detailed_sim_results}

\subsection{Jina Embedding Data Generating Mechanism}

\begin{figure}[h!]
    \centering
    \includegraphics[width=\linewidth]{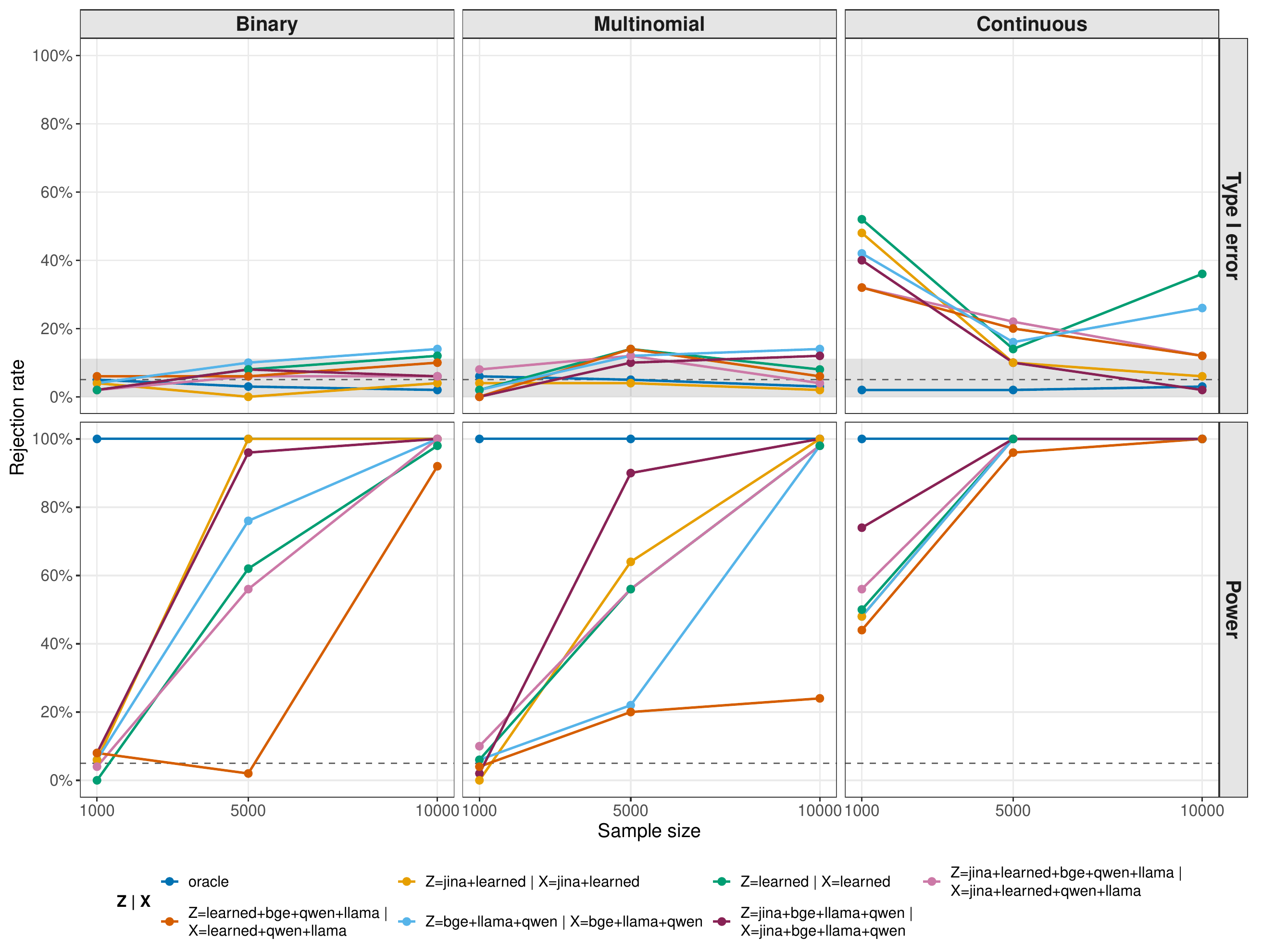}
    \caption{The rejection rates of the eCIT for binary, multinomial and continuous $Y$ generated under the null (first row) and alternative (second row) from the Jina embeddings $Z^{Jina}$ and $X^{Jina}$. The eCIT is applied with the Jina embeddings (oracle), the specified diluted embeddings, including additionally learned, BGE, Qwen, and Llama embeddings, as well as the same combinations excluding the Jina embedding. For the PCM, the nuisance regressions on the training dataset are $L_2$-penalized and those on the evaluation set unpenalized; for categorical $Y$ the regressions of $Y$ are multinomial logistic and those of the fitted probabilities and of $\hat f$ are least squares (Appendix~\ref{app:catpcm:test}).  The dashed line marks the nominal level $\alpha=0.05$, the shaded area marking two MC standard errors around it.}
    \label{fig:type1_power_oracle_dilute_concat_jina_binary_multinomial_continuous}
\end{figure}

\begin{figure}
    \centering
    \includegraphics[width=0.9\linewidth]{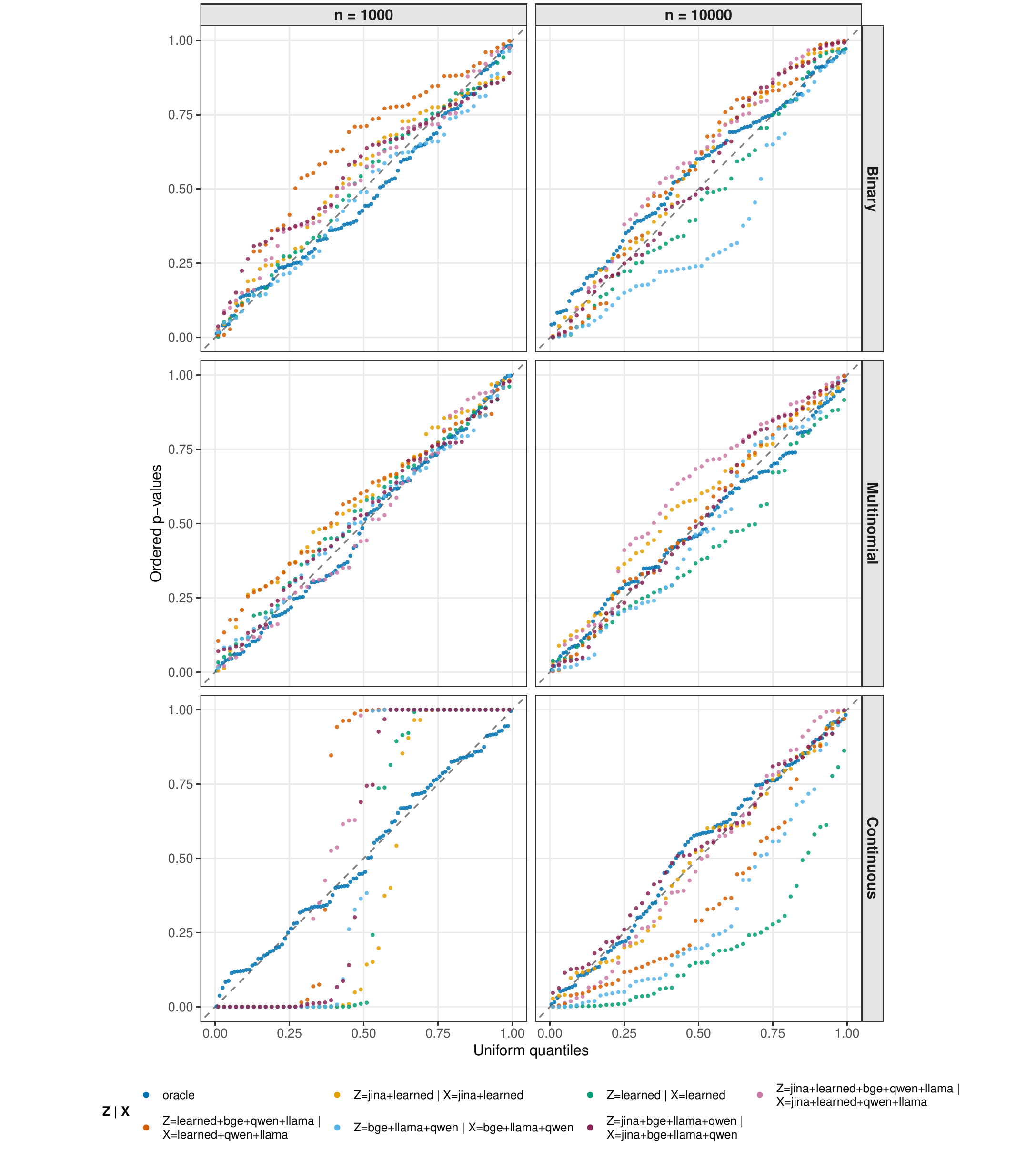}
    \caption{QQ-plots of the observed p-values against the theoretical quantiles of a uniform distribution on the interval $[0,1]$ for eCITs with the specified embeddings. 
    For the PCM, the nuisance regressions on the training dataset are $L_2$-penalized and those on the evaluation set unpenalized; for categorical $Y$ the regressions of $Y$ are multinomial logistic and those of the fitted probabilities and of $\hat f$ are least squares (Appendix~\ref{app:catpcm:test}).
    $Y$ is generated under the null hypothesis from the Jina embedding $Z^{Jina}$ at selected sample sizes (indicated in the panel titles). 
    The diagonal black lines $y=x$ serve as references for the theoretical quantiles of the uniform distribution on $[0,1]$. If the p-values are uniformly distributed, they should align along this line. 
    }
    \label{fig:qq_oracle_dilute_concat_jina_binary_multinomial_continuous}
\end{figure}
\clearpage

\subsection{Qwen Embedding Data Generating Mechanism}

\begin{figure}[h!]
    \centering
    \includegraphics[width=\linewidth]{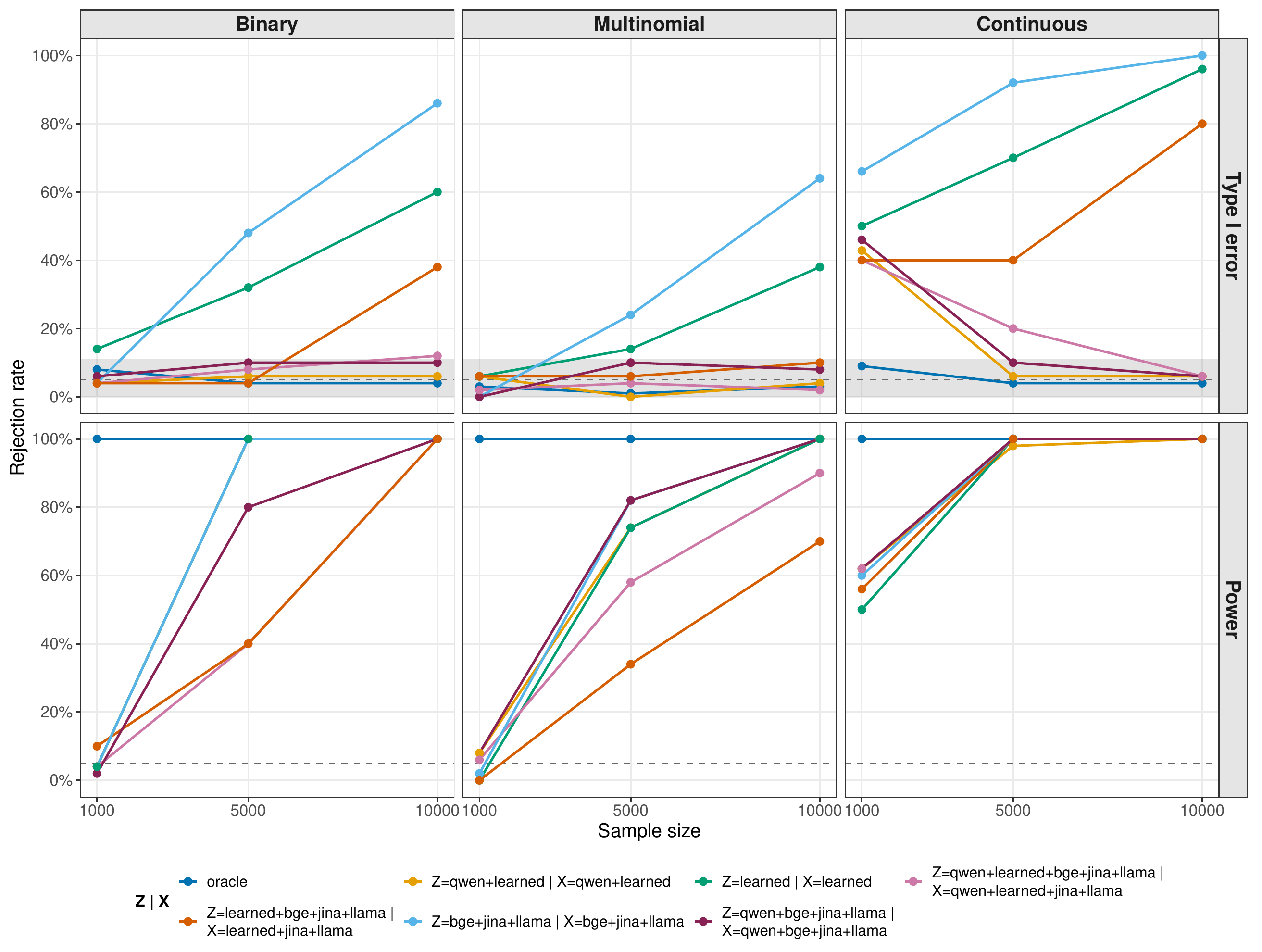}
    \caption{The rejection rates of the eCIT for binary, multinomial and continuous $Y$ generated under the null (first row) and alternative (second row) from the Qwen embeddings $Z^{Qwen}$ and $X^{Qwen}$. The eCIT is applied with the Qwen embeddings (oracle), the specified diluted embeddings, including additionally learned, BGE, Jina, and Llama embeddings, as well as the same combinations excluding the Qwen embedding. For the PCM, the nuisance regressions on the training dataset are $L_2$-penalized and those on the evaluation set unpenalized; for categorical $Y$ the regressions of $Y$ are multinomial logistic and those of the fitted probabilities and of $\hat f$ are least squares (Appendix~\ref{app:catpcm:test}).  The dashed line marks the nominal level $\alpha=0.05$, the shaded area marking two MC standard errors around it.}
    \label{fig:type1_power_oracle_dilute_concat_qwen_binary_multinomial}
\end{figure}

\begin{figure}
    \centering
    \includegraphics[width=0.9\linewidth]{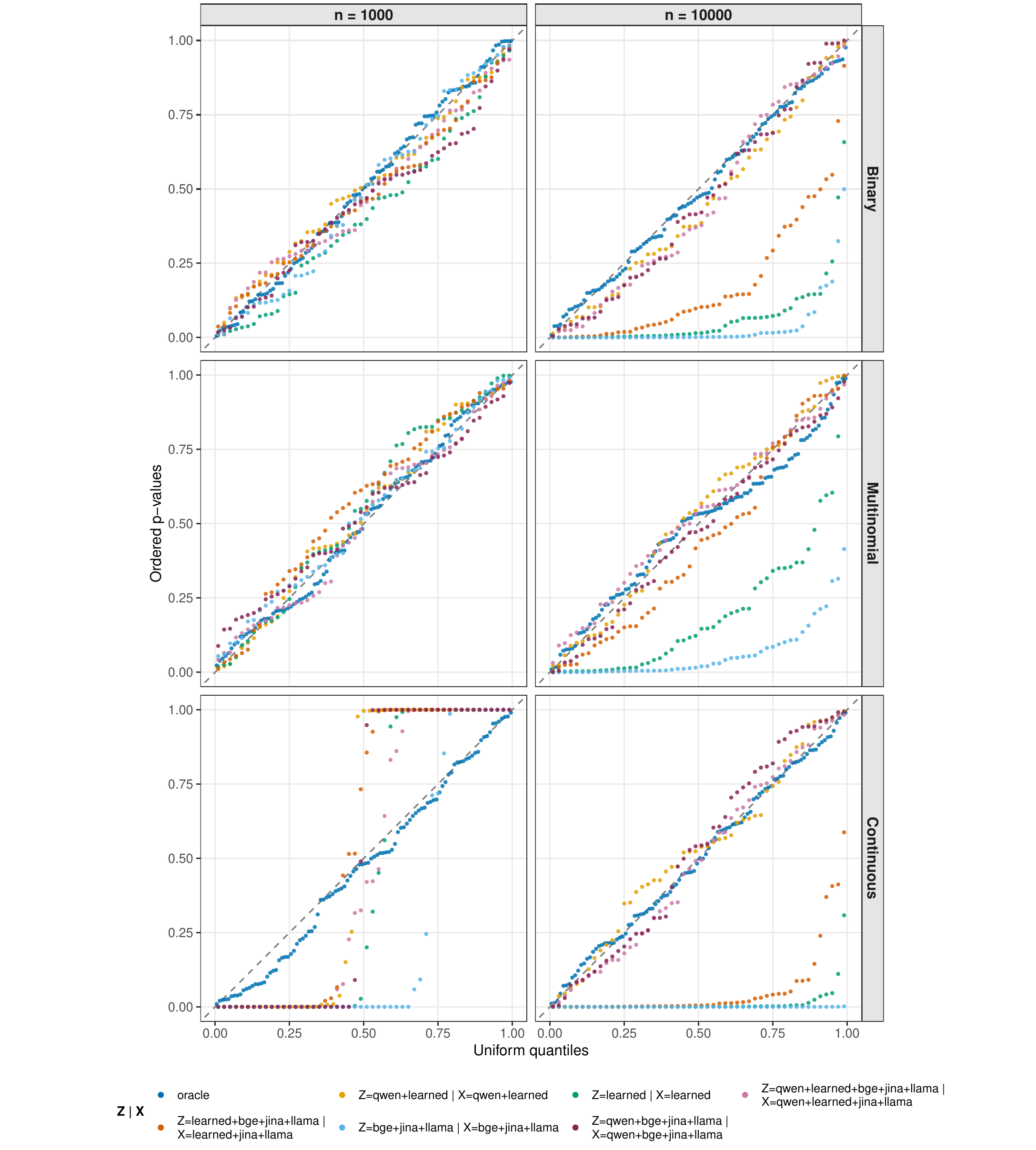}
    \caption{QQ-plots of the observed p-values against the theoretical quantiles of a uniform distribution on the interval $[0,1]$ for eCITs with the specified embeddings. For the PCM, the nuisance regressions on the training dataset are $L_2$-penalized and those on the evaluation set unpenalized; for categorical $Y$ the regressions of $Y$ are multinomial logistic and those of the fitted probabilities and of $\hat f$ are least squares (Appendix~\ref{app:catpcm:test}). $Y$ is generated under the null hypothesis from the Qwen embedding $Z^{Qwen}$ at selected sample sizes (indicated in the panel titles). The diagonal black lines $y=x$ serve as references for the theoretical quantiles of the uniform distribution on $[0,1]$. If the p-values are uniformly distributed, they should align along this line.
    }
    \label{fig:qq_oracle_dilute_concat_qwen_binary_multinomial_continuous}
\end{figure}

\clearpage

\providecommand{\E}{\mathbb{E}}
\providecommand{\Var}{\operatorname{Var}}
\providecommand{\diag}{\operatorname{diag}}
\providecommand{\lmin}{\lambda_{\min}}

\section{The PCM for Categorical \texorpdfstring{$Y$}{Y}}
\label{app:catpcm}

We suppress the embedding maps throughout this appendix and write $X \in \mathcal{X} \subseteq \mathbb{R}^q$, $Z \in \mathcal{Z} \subseteq \mathbb{R}^p$ for $X^{\omega_X}$, $Z^{\omega_Z}$, so that $p = \dim(\mathcal{Z})$. Let $Y$ take $L$ values, drop a baseline class and let $\tilde{Y} \in \{0,1\}^K$, $K = L-1$, be the reduced vector of class indicators of Remark~\ref{remark:equivalence_ci_mean_cat_Y}. 
Write
\[
  g(x,z) = \E[\tilde{Y} \mid X = x, Z = z], \qquad m_Y(z) = \E[\tilde{Y} \mid Z = z],
  \qquad h = g - m_Y,
\]
so that $g$ is the reduced conditional probability mass function of $Y$ and $m_Y$ is the population version of the regression $m_i$ of Subsection~\ref{subsec:cond_mean_indep}, and let $\Sigma(x,z) = \Var(\tilde{Y} \mid X = x, Z = z) = \diag\{g(x,z)\} - g(x,z)g(x,z)^{\top}$.

\subsection{The Categorical PCM}
\label{app:catpcm:test}

Following \citet{lundborg2024projected}, we split the sample into halves $D_1$, $D_2$ of sizes $n_1$, $n_2$. On $D_1$ we fit $\hat{g}$, a multinomial ($L_2$-penalized) logistic regression of $Y$ on $(X,Z)$, and write $\hat{g}(x,z) \in \mathbb{R}^K$ for its vector of fitted probabilities with the baseline class dropped, so that $\hat{g}$ estimates $g$. Then we fit $\tilde{m}$, a column-wise ($L_2$-penalized) least-squares regression of the fitted probabilities $\hat{g}$ on $Z$, and set $\hat{h} = \hat{g} - \tilde{m}$ and
\begin{equation}
  \label{eq:catpcm:direction}
  \hat{f}(x,z) = \bigl(\hat{\Sigma}(x,z) + \hat{c}\, I_K\bigr)^{-1} \hat{h}(x,z) \in \mathbb{R}^{K},
\end{equation}
where $\hat{\Sigma}$ is $\Sigma$ evaluated at $\hat{g}(x,z)$, the entries of the full fitted probability vector being floored at $\varepsilon = 10^{-6}$ and renormalized before the baseline class is dropped.

The ridge penalty $\hat{c}$ is selected on $D_1$. It is needed because class probabilities close to $0$ or $1$, as under quasi-complete separation, leave $\hat{\Sigma}$ nearly singular and the inverse in \eqref{eq:catpcm:direction} unstable. We therefore split $D_1$ into $D_{1a}$ and $D_{1b}$, fit $\hat{g}$ and $\tilde{m}$ on $D_{1a}$, and for each candidate value of $c$ on a logarithmically spaced grid in $[10^{-6}, 1]$ form the direction \eqref{eq:catpcm:direction} and the corresponding scores $\ell_i$ of \eqref{eq:catpcm:stat} on $D_{1b}$, with the nuisance regressions $\hat{m}_Y$ and $\hat{m}_f$ fitted on $D_{1b}$ as well. We select the candidate maximizing $\bar{\ell}/\hat{\sigma}_\ell$ on $D_{1b}$ and refit $\hat{g}$, $\tilde{m}$ and the resulting direction on the whole of $D_1$ at the selected $\hat{c}$. Since $\hat{c}$ is a function of $D_1$ alone, $\hat{f}$ remains $D_1$-measurable and the selection affects the power of the test but not the argument for its level.

On $D_2$ we fit $\hat{m}_Y$, an unpenalized multinomial logistic regression of $Y$ on $Z$, again reduced to its $K$ non-baseline fitted probabilities, and $\hat{m}_f$, an unpenalized least-squares regression of $\hat{f}$ on $Z$, and form $\hat{r}^Y_i = \tilde{Y}_i - \hat{m}_Y(Z_i)$ and $\hat{r}^f_i = \hat{f}(X_i,Z_i) - \hat{m}_f(Z_i)$. The test statistic is
\begin{equation}
  \label{eq:catpcm:stat}
  T = \sqrt{n_2}\,\frac{\bar{\ell}}{\hat{\sigma}_\ell}, \qquad
  \ell_i = \langle \hat{r}^Y_i, \hat{r}^f_i \rangle, \quad
  \bar{\ell} = \frac{1}{n_2}\sum_{i \in D_2} \ell_i, \quad
  \hat{\sigma}_\ell^2 = \frac{1}{n_2}\sum_{i \in D_2} (\ell_i - \bar{\ell})^2,
\end{equation}
and we reject at level $\alpha$ when $T > z_{1-\alpha}$. This is Algorithm~1 of \citet{lundborg2024projected}, as implemented for continuous $Y$ in the \textsf{R} package
\texttt{comets} \citep{kook2024algorithm}, with their scalar direction replaced by \eqref{eq:catpcm:direction} and their product $\hat{r}^Y_i \hat{r}^f_i$ by an inner product. The ridge constant $\hat{c}$ has no counterpart there, and is required here to invert $\hat{\Sigma}$ stably.

\subsection{Validity}
\label{app:catpcm:validity}

The summands $\ell_i$ in \eqref{eq:catpcm:stat} are scalar, since the $K$-vector residuals are projected onto one dimension by the inner product before averaging, so the numerator remains an average of real-valued summands that are conditionally independent and, under the null, conditionally mean-zero given $D_1$. The argument of \citet[Theorem~4]{lundborg2024projected} is therefore available with their product $\epsilon_i \xi_i$ replaced by $\langle \epsilon_i, \xi_i \rangle$, provided the conditions it places on the nuisance estimates are read in vector form. We record these conditions below; a detailed verification, including vector restatements of the supporting lemmas, is left to future work.

\begin{assumption}
  \label{ass:catpcm}
  Let $\epsilon = \tilde{Y} - g(X,Z)$, $m_f(z) = \E[\hat{f}(X,Z) \mid Z = z, D_1]$, $\sigma^2 = \Var(\langle \epsilon, \hat{f}(X,Z) - m_f(Z)\rangle \mid D_1)$, $E_1 = \E[\lVert \hat{m}_Y(Z) - m_Y(Z) \rVert_2^2 \mid D_1]$ and $E_2 = \sigma^{-2}\,\E[\lVert \hat{m}_f(Z) - m_f(Z) \rVert_2^2 \mid D_1]$. Assume (a) $E_1 E_2 = o_{\mathbb{P}}(n_2^{-1})$; (b) $\sigma^{-(2+\delta)}\,\E[\lVert \hat{f}(X,Z) - m_f(Z)\rVert_2^{2+\delta} \mid D_1] = O_{\mathbb{P}}(1)$ for some $\delta > 0$; and (c) $\lmin(\Sigma(X,Z)) \geq c_0$ almost surely for some $c_0 > 0$.
\end{assumption}

Two features of the categorical case are worth noting. First, by Remark~\ref{remark:equivalence_ci_mean_cat_Y} the conditional mean of $\tilde{Y}$ is the conditional distribution of $Y$, so conditional mean independence of $\tilde{Y}$ and $X \indep Y \mid Z$ are the same statement and the mean-level null class coincides with the CI null class for the i.i.d.\ laws considered in Corollary~\ref{cor:mean-iid}. 
The PCM therefore tests the CI hypothesis itself here, and nothing is lost by using a CMIT. 
Second, $\tilde{Y}$ is bounded, which makes the moment condition (b) a statement about $\lVert \hat{f} - m_f \rVert_2$ alone. Condition (c), by contrast, requires every class probability to be bounded away from $0$ and $1$, which is a substantive restriction when $L$ is large and some classes are rare.

We do not claim uniform asymptotic level over a nonparametric class for this variant. Null calibration is instead assessed by simulation in Subsection~\ref{subsec:simulation} and Appendix~\ref{app:detailed_sim_results}.

\subsection{The Precision-Weighted Direction}
\label{app:catpcm:optimal}

The weighting in \eqref{eq:catpcm:direction} is motivated by the following, which identifies the direction maximizing the population signal-to-noise ratio. At $K = 1$ it gives $f \propto h/v$ with $v = \Var(\tilde{Y} \mid X, Z)$.

\begin{proposition}
  \label{prop:catpcm}
  Suppose $\Sigma(X,Z)$ is almost surely positive definite with $0 < \E[h^{\top}\Sigma^{-1}h] < \infty$, and let $\mathcal{F}$ be the set of measurable $f : \mathcal{X} \times \mathcal{Z} \to \mathbb{R}^K$ with $0 < \E[f^{\top}\Sigma f] < \infty$, all functions being evaluated at $(X,Z)$. Then
  \begin{equation}
    \label{eq:catpcm:snr}
    \sup_{f \in \mathcal{F}} \frac{\bigl(\E\langle h, f\rangle\bigr)^2}{\E\bigl[f^{\top}\Sigma f\bigr]} = \E\bigl[h^{\top}\Sigma^{-1}h\bigr],
  \end{equation}
  and the supremum is attained at $f \in \mathcal{F}$ if and only if $f = \kappa\,\Sigma^{-1}h$ almost surely for some $\kappa > 0$.
\end{proposition}

\begin{proof}
  Let $\Sigma^{1/2}$ be the symmetric positive definite square root of $\Sigma$ and put $u = \Sigma^{-1/2}h$, $w = \Sigma^{1/2}f$, which are measurable since $A \mapsto A^{\pm 1/2}$ is continuous on the open cone of positive definite matrices, and satisfy $\E\lVert u \rVert_2^2 = \E[h^{\top}\Sigma^{-1}h]$ and $\E\lVert w \rVert_2^2 = \E[f^{\top}\Sigma f]$, both in $(0,\infty)$. Pointwise, $\langle h, f \rangle = \langle u, w \rangle \leq \lVert u \rVert_2 \lVert w \rVert_2$ almost surely, so taking expectations and applying Cauchy--Schwarz in $L^2(\mathbb{P})$,
  \begin{equation}
    \label{eq:catpcm:cs}
    \E\langle h, f \rangle \;\leq\; \E\bigl[\lVert u \rVert_2 \lVert w \rVert_2\bigr] \;\leq\; \bigl(\E[h^{\top}\Sigma^{-1}h]\bigr)^{1/2}\bigl(\E[f^{\top}\Sigma f]\bigr)^{1/2}.
  \end{equation}
  Since $f$ and $-f$ share the denominator in \eqref{eq:catpcm:snr}, we may assume $\E\langle h, f\rangle > 0$; squaring and dividing bounds the left-hand side of \eqref{eq:catpcm:snr} by $\E[h^{\top}\Sigma^{-1}h]$, with equality at $f = \Sigma^{-1}h \in \mathcal{F}$.

  If $f \in \mathcal{F}$ attains the supremum, both inequalities in \eqref{eq:catpcm:cs} are equalities. The first forces $\langle u, w \rangle = \lVert u \rVert_2 \lVert w \rVert_2$ almost surely, the second $\lVert w \rVert_2 = \kappa \lVert u \rVert_2$ almost surely for some $\kappa \in [0,\infty)$, with $\kappa > 0$ since $\E\lVert w \rVert_2^2 > 0$. On $\{u \neq 0\}$ the equality case in $\mathbb{R}^K$ gives $w = \kappa u$, and on $\{u = 0\}$ we have $w = 0$, so $\Sigma^{1/2} f = \kappa \Sigma^{-1/2} h$ almost surely. The converse is the computation above.
\end{proof}

\begin{remark}
  \label{rem:catpcm:ridge}
  The constant $\hat{c}$ in \eqref{eq:catpcm:direction} serves only to invert $\hat{\Sigma}$ stably when fitted probabilities are close to the boundary, and has no counterpart in   Algorithm~1 of \citet{lundborg2024projected}, whose Step~2(ii) calibrates a constant of the   same name against a variance identity. We apply no $\mathrm{sign}(\rho)$ correction. Since   $\hat{\Sigma} + \hat{c}I_K$ is positive definite, $\langle \hat{h}_i, \hat{f}_i \rangle \geq 0$    by construction.
\end{remark}

\subsection{Nuisance Estimation on the Evaluation Half}
\label{app:catpcm:rates}

The evaluation-half nuisances are fitted unpenalized and in sample, which yields an exact finite-sample orthogonality. Let $W \in \mathbb{R}^{n_2 \times (p+1)}$ have rows $W_i = (1, Z_i^{\top})$ and collect the residuals in $\hat{R}^Y, \hat{R}^f \in \mathbb{R}^{n_2 \times K}$. The score equations of the multinomial logistic fit give $W^{\top}\hat{R}^Y = 0$ and the normal equations of the least-squares fit give $W^{\top}\hat{R}^f = 0$. Since $\hat{m}_f(Z_i) = W_i \hat{B}$ for some $\hat{B} \in \mathbb{R}^{(p+1) \times K}$, the first identity yields
\begin{equation}
  \label{eq:catpcm:orth}
  \sum_{i \in D_2} \langle \hat{r}^Y_i, \hat{r}^f_i \rangle = \sum_{i \in D_2} \langle \hat{r}^Y_i, \hat{f}(X_i,Z_i) \rangle,
\end{equation}
so the regression of $\hat{f}$ on $Z$ does not contribute to the numerator. Under the null, writing $\hat{r}^Y_i = \epsilon_i - M_i$ with $M_i = \hat{m}_Y(Z_i) - m_Y(Z_i)$, the numerator splits as $\sum_{i \in D_2} \langle \epsilon_i, \hat{r}^f_i \rangle - \sum_{i \in D_2} \langle M_i, \hat{r}^f_i \rangle$, in which the first term is conditionally mean-zero given $(X,Z)^{n_2}$ and $D_1$, because $\hat{r}^f$ does not depend on $Y$ on $D_2$, leaving the second as the only bias term.

A Cauchy--Schwarz bound on that term is of order $\sqrt{n_2}\,(E_1E_2)^{1/2}$, which is negligible under Assumption~\ref{ass:catpcm}(a) and requires the two nuisance errors to be small relative to $n_2^{-1/2}$ jointly. Each evaluation-half nuisance is an unpenalized fit with $K(p+1)$ free parameters, and in Subsection~\ref{subsec:real_world_application} the conditioning sets consist of high-dimensional embeddings with $p$ between $768$ and $1888$ at $n_2$ of roughly $25{,}000$ to $27{,}000$, so $p$ is not small relative to $\sqrt{n_2}$ and we cannot appeal to this bound there. 
It is sufficient rather than necessary, and conservative in that it ignores the orthogonality \eqref{eq:catpcm:orth}. The simulations in Subsection~\ref{subsec:simulation} use conditioning sets of comparable dimension at smaller $n_2$, hence at a less favorable ratio, and the arms with sufficient $Z^{\omega_Z}$ remain calibrated there, which suggests that the bound overstates the requirement at these dimensions. 
Misspecification of $m_Y$ on $Z$ is the more serious concern, since it is not addressed by the orthogonality. The simulation does not separate it from mean insufficiency: in the arms retaining the generating embedder the outcome is a link applied to a linear score in that embedding, so the linear nuisances are close to correctly specified by construction, while in the arms omitting it the conditional mean given the supplied embedding is in general no longer linear, so misspecification and mean insufficiency act together and their contributions to the observed inflation cannot be attributed separately.

\vskip 0.2in
\bibliography{references}

\end{document}